\documentclass{article}

\PassOptionsToPackage{numbers, compress}{natbib}
\usepackage[final]{neurips_2022}

\usepackage[utf8]{inputenc} %
\usepackage[T1]{fontenc}    %
\usepackage{hyperref}       %
\usepackage{url}            %
\usepackage{booktabs}       %
\usepackage{amsfonts}       %
\usepackage{nicefrac}       %
\usepackage{microtype}      %
\usepackage[dvipsnames]{xcolor}
\usepackage{xspace}

\usepackage{graphicx}
\usepackage{subfigure}
\usepackage{bbding}
\usepackage{bm}
\usepackage{enumitem}
\usepackage{footnote}
\usepackage{multirow}

\usepackage{titlesec}
\titlespacing{\section}{0.9\parskip}{0.9\parskip}{0.9\parskip}

\definecolor{BOME}{RGB}{31, 119, 180}
\definecolor{BSG1}{RGB}{44, 160, 44}
\definecolor{BVFSM}{RGB}{255, 127, 14}
\definecolor{Penalty}{RGB}{214, 39, 40}

\usepackage{hyperref}

\usepackage{qiangstyle}

\title{BOME! Bilevel Optimization Made Easy: \\  A Simple First-Order Approach}

\author{%
  Mao Ye$^1$\thanks{Equal contribution. MY mainly contributes on developing theory and BL mainly contributes on conducting experiment. Both authors contribute equally on paper writing.}~~~~~Bo Liu$^1$$^*$~~~~~Stephen Wright$^2$~~~~~Peter Stone$^1$~~~~~Qiang Liu$^1$ \\
  $^1$The University of Texas at Austin ~~~~~ $^2$University of Wisconsin-Madison\\
  $^1$\texttt{\{my21,bliu,pstone,lqiang\}@cs.utexas.edu},~~~~~$^2$\texttt{swright@cs.wisc.edu}\\
}

\begin{document}
\long\def\commentp#1{{\bf **Peter: #1**}}

\global\long\def\th{\theta}%
\global\long\def\d{\delta^{*}}%
\global\long\def\hdq{\hat{\nabla}{q}}%
\global\long\def\lm{\lambda^{*}}%
\global\long\def\bd{\bar{d}}%
\global\long\def\hk{\hat{\mathcal{K}}}%
\global\long\def\K{\mathcal{K}}%
\global\long\def\ip{p}%

\global\long\def\ul{\eta}%
\global\long\def\uu{h }%

\global\long\def\bi{b_{1}}%
\global\long\def\bii{b_{2}}%
\global\long\def\biii{b_{3}}%
\global\long\def\biv{b_{4}}%
\global\long\def\bv{b_{5}}%
\global\long\def\bvi{b_{6}}%
\global\long\def\bviii{b_{7}}%

\newtheorem{theorem}{Theorem}
\newtheorem{lemma}{Lemma}
\newtheorem{proposition}{Proposition}
\newtheorem{assumption}{Assumption}
\newtheorem{definition}{Definition}

\newcommand{\sw}[1]{{\color{purple}{[\textbf{Steve:} #1}]}}
\newcommand{\mao}[1]{{\color{green}{[\textbf{Mao:} #1}]}}

\maketitle

\begin{abstract}
Bilevel optimization (BO) is useful for solving a variety of important machine learning problems including but not limited to hyperparameter optimization, meta-learning, continual learning, and reinforcement learning.
Conventional BO methods need to differentiate through the low-level optimization process with implicit differentiation, which requires expensive calculations related to the Hessian matrix. There has been a recent quest for first-order methods for BO,  but the methods proposed to date tend to be complicated and impractical for large-scale deep learning applications. In this work, we propose a simple first-order BO  algorithm that depends only on first-order gradient information, requires no implicit differentiation, and is practical and efficient for large-scale non-convex functions in deep learning. We provide non-asymptotic convergence analysis of the proposed method to stationary points for non-convex objectives and present empirical results that show its superior practical performance.
\end{abstract}

\vspace{-10pt}
\section{Introduction}
\vspace{-5pt}
We consider the bilevel optimization (BO) problem: 
\begin{equation}
    \min_{v,\theta} f\big(v, \theta\big)~~~~~s.t.~~~~~\theta \in \argmin_{\theta'} g\big(v, \theta'\big),
    \label{eq:bo}
\end{equation}
where the goal is to minimize an \emph{outer objective} $f$ whose variables include the solution of another minimization problem w.r.t an \emph{inner objective} $g$. 
The $\theta$ and $v$ are the \emph{inner} and \emph{outer} variables, respectively. We assume that $v \in \RR^m, \theta \in \RR^n$ and that
$g(v,\cdot)$ attains a minimum for each $v$.

BO is useful in a variety of machine learning tasks. 
A canonical example is hyperparameter optimization, 
in which case $f$ (resp. $g$) is the validation (resp. training) loss associated with a model parameter $\theta$ and 
a hyperparameter $v$, and we want to find the optimal hyperparameter $v$ to minimize the validation loss $f$ when $\theta$ is determined by minimizing the training loss; see e.g., \citet{pedregosa2016hyperparameter, franceschi2018bilevel}. 
Other applications include meta learning~\citep{franceschi2018bilevel}, continual learning~\citep{pham2020contextual}, reinforcement learning~\citep{yang2019provably}, and adversarial learning~\citep{jiang2021learning}. %
See \citet{liu2021investigating} for a recent survey.

BO is notoriously challenging due to its nested nature.
Despite the large literature,  most existing methods for BO are slow and unsatisfactory in various ways. 
For example, a major class of BO methods is based on direct gradient descent on the outer variable $v$ while viewing   the optimal inner variable $\theta^*(v) = \arg \min_\theta g(v,\theta)$  as a (uniquely defined) function of $v$.  
The key difficulty is to calculate the derivative $\dthetadv$ which may require expensive manipulation of the Hessian matrix of $g$ via the implicit differentiation theorem.
Another approach is to replace the low level optimization with the stationary condition $\dd_\theta g(v, \theta) =0$. 
This still requires Hessian information,
and more importantly, is unsuitable for nonconvex $g$ since it allows $\theta$ to be any stationary point of  $g(v, \cdot)$, not necessarily a minimizer.
To the best of our knowledge,  
the only existing fully first-order BO algorithms\footnote{By fully first-order, we mean methods that only require information of $f,g,\dd f,\dd g$, so this excludes methods that apply auto-differentiation or conjugate gradient that need multiple steps of matrix-vector computation.} are BSG-1~\citep{giovannelli2021bilevel} and
 BVFSM with its variants~\citep{liu2021value, liu2021valueseq, liu2021towards}; but BSG-1 relies on a non-vanishing approximation that does not yield convergence to the correct solution in general, and BVFSM is sensitive to hyper-parameters on large-scale practical problems and lacks a complete non-asymptotic analysis for the practically implemented algorithm. 

In this work, we seek a \emph{simple and fast fully first-order} BO method that can be used with non-convex functions including those appear in deep learning applications.
The idea is to reformulate \eqref{eq:bo} as a single-level constrained optimization problem 
using  the so-called value-function-based approach~\citep{dinh2010subdifferentials,dempe2020bilevel}. The constrained problem is then solved by stopping gradient on the single variable that contains the higher-order information and applying a simple first-order dynamic barrier gradient descent method based on a method of \citet{gong2021automatic}. Our contributions are: \textbf{1)} we introduce a novel and fast BO method by applying a modified dynamic barrier gradient descent on the value-function reformulation of BO; \textbf{2)} Theoretically, we establish the non-asymptotic convergence of our method to local stationary points (as measured by a special KKT loss) for non-convex $f$ and $g$. Importantly, to the best of our knowledge, this work is the first to establish non-asymptotic convergence rate for a fully first-order BO method.
This result is also much beyond that of \citet{gong2021automatic} and \citet{ji2021bilevel}. \textbf{3)} Empirically, the proposed method achieves better or comparable performance while being more efficient than state-of-the-art BO methods on a variety of benchmarks.

\vspace{-5pt}
\section{Background}
\vspace{-5pt}
\label{sec:background}
This section provides a brief background on traditional BO methods. Please see \citet{bard2013practical, dempe2020bilevel, dempe2002foundations} for overviews, and
\citet{liu2021investigating} for a survey on recent ML applications.

\textbf{Hypergradient Descent} 
Assume that the minimum of $g(v, \cdot)$ is unique for all $v$ 
so that we can  write $\theta^*(v) = \argmin_{\theta} g(v,\theta)$ as a function of $v$; 
this is known as the low-level singleton (LLS) assumption.   
The most straightforward approach to solving \eqref{eq:bo} is to conduct gradient descent on $f(\v, \sv)$ as a function of $v$. 
Note that 
$$
    \dd_v f(v,\th^*(v))
    = \ddv f(\v, \sv) + \textcolor{RedViolet}{\dthetadv}   \ddt f(\v, \sv).
$$
The difficulty is to compute $\textcolor{RedViolet}{\dthetadv}$. From implicit function theorem, it satisfies a linear equation:
\begin{equation}
    \label{eq:lineareq}
    \ddvt g(\v, \sv) + \ddtt g(\v, \sv) \textcolor{RedViolet}{\dthetadv} = 0.
\end{equation}
If $\ddtt g$ is invertible, we can solve for $\textcolor{RedViolet}{\dthetadv}$ and 
obtain a gradient update rule on $v$:  
\begin{equation*} 
\v_{k+1} \gets \v_k - \xi \left (
    \ddv f_k - \big(\ddvt g_k\big)^\top \big(\ddtt g_k \big)^{-1} \ddt f_k 
    \right ), 
\end{equation*}
where $k$ denotes iteration, $\dd_1 f_k = \dd_1 f(\v_k, \ss^*(\v_k))$ 
and similarly for the other terms. 
This approach is sometimes known as 
the \emph{hypergradient descent}. However, hypergradient descent is computationally expensive: Besides requiring evaluation of the inner optimum $\theta^*(v_k)$,  
the main computational bottleneck is to solve the linear equation in  \eqref{eq:lineareq}.   
Methods have been developed that  approximate \eqref{eq:lineareq} using conjugate gradient~\citep{pedregosa2016hyperparameter,rajeswaran2019meta,grazzi2020iteration}, Neumann series~\citep{liao2018reviving,lorraine2020optimizing}, and related variants \citep{ghadimi2018approximation}.  
Another popular approximation approach is to replace $\dthetadv$ with $\dd_v \th^{(T)}(\v)$, where  $\th^{(T)}(v)$ denotes the $T$-th iteration of  gradient  descent or other optimization steps on $g(v,\th)$ w.r.t. $\th$ starting from certain initialization.
The gradient $\dd_v \th^{(T)}(\v)$ can be calculated with 
auto-differentiation (AD) with either 
forward mode~\citep{franceschi2017forward}, backward mode~\citep{franceschi2018bilevel,franceschi2017forward,shaban2019truncated,li2021fully,arbel2021amortized} or their variants \citep{liu2020generic}. While these approaches claim to be first-order, they require many Hessian-vector or Jacobian-vector products at each iteration and are slow for large problems. 

Other examples of approximation methods include a 
neural surrogate method which approximates $\theta^*(v)$ and its gradient  $\dthetadv$ with neural networks~\citep{mackay2019self} and
Newton-Gaussian approximation of the Hessian matrix with covariance of gradient \citep{giovannelli2021bilevel}.
Both approaches  
introduce non-vanishing approximation error that is difficult to control.
The neural surrogate method also suffers from high training cost for the neural network.

\textbf{Stationary-Seeking Methods.}   
An alternative method is to  replace the argmin  constraint in \eqref{eq:bo} 
with the stationarity condition $\dd_\theta g(\v, \theta) =0,$ yielding a constrained optimization: 
\begin{equation} \label{eq:gbo}
    \min_{\v,\th}  f(\v, \th)~~~~~s.t.~~~~~
    \dd_\th g(\v, \th) = 0. 
\end{equation}
Algorithms for nonlinear equality constrained optimization can then be applied~\citep{mehra2021penalty}. 
The constraint in \eqref{eq:gbo} guarantees only that  $\theta$ is a stationary point of $g(v,\cdot)$, so it 
is  equivalent to \eqref{eq:bo} 
only when $g$ is convex w.r.t. $\theta$. 
Otherwise, the solution of \eqref{eq:gbo} can be a maximum or saddle point of $g$. This 
makes it problematic for deep learning, where non-convex functions are  pervasive.

\vspace{-5pt}
\section{Method}
\vspace{-5pt}
\label{sec:method}
We consider a \emph{value function approach} \citep[see e.g.,][]{outrata1990numerical,ye1995optimality, liu2021value}, 
which yields natural first-order algorithms for non-convex $g$ and requires no computation of Hessian matrices.
It is based on the observation that \eqref{eq:bo} is equivalent to the following constrained optimization (even for non-convex $g$): 
\begin{equation}
    \min_{v, \theta} \, f(v, \th)~~~~~{s.t.}~~~~~
    q(v, \th) := g(v,\th) - g^*(v) \leq 0, 
\label{eq:bo_constrain}
\end{equation}
where $
 g^*(v) := \min_{\th} g(v, \th) =  g(v, \th^*(v))$ is known as the value function.  
Compared with the {hypergradient approach}, this formulation 
\textbf{does not require calculation of the implicit derivative $\nabla_v \th^*(v)$}: Although $g^*(v)$ depends on $\vs$, its derivative $\nabla_v g^*(v)$ does not depend on  $\nabla_v \th^*(v)$, by Danskin's theorem: \bbb  
\nabla_v g^*(v) = \nabla_1 g(v, \vs) + 
 \dd_v \vs  \nabla_2 g(v, \vs) = \nabla_1 g(v, \vs),  \label{eq:dfd}
\eee  
where the second term in \eqref{eq:dfd} vanishes because we have  $\nabla_2 g(v, \vs) = 0$ by definition of the optimum $\vs$. 
Therefore, provided that we can evaluate $\vs$ at each iteration, solving \eqref{eq:bo_constrain}   yields an  algorithm for BO that requires no Hessian computation. 
In this work, we make use of the dynamic barrier gradient descent algorithm  of \citet{gong2021automatic} to solve \eqref{eq:bo_constrain}.
This is an elementary first-order algorithm for solving constrained optimization, but it applies only to a special case of the bilevel problem and must be extended to handle the general case we consider here.

\textbf{Dynamic Barrier Gradient Descent.} 
The idea is to iterative update the parameter $(v,\th)$ 
to reduce $f$  while controlling the decrease of the constraint ${q}$, ensuring that ${q}$ decreases whenever $q>0$. Specifically, denote $\xi$ as the step size, the update at each step is
\bbb 
\label{eq:updadd} 
(\v_{k+1}, \ss_{k+1}) 
\gets (\v_{k}, \ss_{k}) 
- \xi \delta_k,
\eee 
\bbb \label{equ:bar_opt} 
\text{where}~~~\delta_k = \argmin_{\delta} \norm{\dd f(v_k, \th_k) - \delta }^2 
~~\text{s.t.}~~ 
\langle \dd q(v_k, \th_k),  \delta \rangle \geq \phi_k. 
\eee  
Here %
$\dd f_k := \dd_{(\v,\th)} f(v_k, \th_k)$, 
$\dd q_k := \dd_{(\v,\th)} q(v_k, \th_k)$, 
and $\phi_k\geq 0$ is a non-negative control barrier 
and should be strictly positive $\phi_k >0$ in the non-stationary points of $q$: the lower bound on the inner product of $\dd q(v_k,\theta_k)$  and $\delta_k$ ensures that the update in \eqref{eq:updadd} can only decrease $q$ (when step size $\xi$ is sufficiently small) until it reaches stationary. 
In addition, by enforcing $\delta_k$ to be close to $\dd f(v_k, \theta_k)$ in \eqref{equ:bar_opt}, 
we decrease the objective $f$ as much as possible so long as it does not conflict with descent of $q$. 
    
    Two straightforward choices of $\phi_k$ that satisfies the condition above are $\phi_{k} = \eta q(v_k, \th_k)$ and 
    $\phi_k = \eta 
    \norm{\dd q(v_k, \th_k)}^2$ with 
    $\eta > 0$. 
    We find that both choices of $\phi_k$ work well empirically and use $\phi_k = \eta 
    \norm{\dd q(v_k, \th_k)}^2$ as the default (see Section~\ref{sec:observations}). 

The optimization in \eqref{equ:bar_opt} yields a simple closed form solution: 
\begin{align*}
    \delta_{k}=\nabla f(v_k,\th_k)+\lambda_{k}\nabla q(v_k,\th_k), ~~
    \text{with }~~ \lambda_{k}=\max\left(\frac{\phi_{k}-\left\langle \nabla f(v_k,\th_k),\nabla q(v_k,\th_k)\right\rangle }{||\nabla q(v_k,\th_k)||^{2}},~~0\right),
\end{align*}
and $\lambda_k=0$  
in the case of $||\nabla q(v_k,\th_k)||=0$.

\textbf{Practical Approximation.}
The main bottleneck of the method above is to calculate the $q(v_k,\th_k)$ and $\nabla q(v_{k}, \theta_k)$
which requires evaluation of $\th^{*}(v_{k})$. In practice, we approximate
$\th^{*}(v_{k})$ by $\th_{k}^{(T)}$, where $\th_{k}^{(T)}$ is obtained by running $T$ steps of gradient descent of $g(v_k, \cdot)$ w.r.t.
$\th$ starting from $\th_{k}$. That is, we set $\th_{k}^{(0)}=\th_{k}$
and let 
\begin{align}\label{equ:thetaTT}
    \hspace{-0.3cm}
    \th^{(t+1)}_k = \th^{(t)}_k - \alpha \nabla_\th g(v_k, \th^{(t)}_k), ~~~ t = 0,\ldots, T-1, 
\end{align}
for some step size parameter $\alpha>0$. 
We obtain an estimate of $q(v,\th)$ at iteration $k$ by replacing $\th^{*}(v_{k})$ with $\th_{k}^{(T)}$:  %
$
\hat{q}(v,\th)=g(v,\th)-g(v,\textcolor{black}{\th^{(T)}_k}).
$

We substitute $\hat{q}(v_k,\th_k)$  into (\ref{equ:bar_opt}) to obtain the update direction $\delta_k$. 
The full procedure is summarized in Algorithm~\ref{alg:bome}.
Note that the $\textcolor{black}{\th^{(T)}_k}$ is viewed as a constant when defining $\hat{q}(v,\th)$ and hence no differentiation of $\th^{(T)}_k$ is performed when calculating the gradient  $\dd \hat{q}$. This differs from {truncated back-propagation methods} \citep[e.g.,][]{shaban2019truncated} which differentiate through $\theta^{(T)}_k$ as a function of $v$. 
{
Alternatively, it can be viewed as a plug-in estimator. We know that
\begin{align*}
\nabla_{v_{k}}q(v_{k},\th_{k}) & =\nabla_{v_{k}}g(v_{k},\th_{k})-\nabla_{v_{k}}g(v_{k},\th^{*}(v_{k}))\\
 & =\nabla_{v_{k}}g(v_{k},\th_{k})-\left[\nabla_{1}g(v_{k},\th^{*}(v_{k}))+\nabla_{v_{k}}\th^{*}(v_{k})\nabla_{2}g(v_{k},\th^{*}(v_{k}))\right]\\
 & =\nabla_{v_{k}}g(v_{k},\th_{k})-\nabla_{1}g(v_{k},\th^{*}(v_{k})),
\end{align*}
where $\nabla_{1}$ denotes taking the derivative w.r.t. the first variable.
Since $\th^{*}(v_{k})$ is unknown, we estimate $\nabla_{1}g(v_{k},\th^{*}(v_{k}))$ by plugging-in $\th_{k}^{(T)}$ to approximate $\th^{*}(v_{k})$:
\[
\nabla_{v_{k}}\hat{q}(v_{k},\th_{k})=\nabla_{v_{k}}g(v_{k},\th_{k})-\nabla_{1}g(v_{k},\th^{*}(v_{k})).
\]
}
Each step of Algorithm~\ref{alg:bome} can be viewed as taking one step (starting from $v_k,\th_k$) toward solving an approximate constrained optimization problem: 
\bbb \label{eq:ghatopt}
\min_{\v, \th} f(\v, \th)
~~~~~s.t.~~~~~ g(\v, \th) \leq g( v, \th^{(T)}_k), 
\eee 
which 
can be viewed as a relaxation of 
the exact constrained optimization formulation
\eqref{eq:bo_constrain}, 
because $\{(\v,\th)\colon g(\v,\th)\leq g^*(v)\}$ is a subset of $\{(\v,\th)\colon g(\v,\th) \leq g(\v, \theta^{(T)}_k)\}$. 

\begin{algorithm*}[t]
    \vspace{-0.04cm}
    \caption{Bilevel Optimization Made Easy (BOME!)}
    \begin{algorithmic}
    \STATE \textbf{Goal}: Solve $\min_{v,\theta} f(v,\theta)$ ~~$s.t.$~~ $\theta \in \argmin g(v, \cdot)$. 
    \STATE \textbf{Input}: Initialization $(v_0, \th_0)$; 
     inner step $T$; outer and inner stepsize $\xi$, $\alpha$ (set $\alpha = \xi$  by default).  
        \vspace{.05em} 
        \FOR{ iteration $k$}  
            \vspace{.1em} 
            \STATE 1. Get $\textcolor{MidnightBlue}{\theta_k^{(T)}}$ by $T$ steps of gradient descent on $g(v_k, \cdot)$ starting from $\th_k$ (See Eq.~\eqref{equ:thetaTT}).
            \STATE 
            2. Set $\hat{q}(v,  \theta)=g(v,\th)-g(v,\textcolor{MidnightBlue}{\th_{k}^{(T)}})$. %
            \vspace{.2em} 
            \STATE 3. Update $(v, \theta): (v_{k+1},\theta_{k+1}) \gets (v_{k},\theta_k) - \xi 
            ( \dd f(v_{k},\theta_k)  + \lambda_k \dd \hat q(v_{k},\theta_k))$ 
            \bb 
             ~~~~~~\text{where}~~~~~
             \lambda_k = %
            \max\left (  \frac{\phi_k - \langle\dd f(v_k, \theta_k),~~ \dd\hat q(v_k,\theta_k)  \rangle  }{\norm{\dd\hat q(v_k,\theta_k) }^2},~~0  \right ),
            \vspace{-.5em} 
            \ee 
            \vspace{-.2em}
        ~~~~and $\phi_k = \eta ||\dd \hat q(v_{k},\theta_k) ||^2$ (default),  or
        $\phi_k = \eta \hat{q}(v_{k},\theta_k) $ with  $\eta>0$. %
                    \vspace{.5\baselineskip}
        \STATE \textbf{Remark}: 
         ~1)  
        We treat $\textcolor{MidnightBlue}{\theta_k^{(T)}}$ as constant when taking derivative of $\hat q$; 
         ~2) In practice, step 3 can have separate stepsize $(\xi_v, \xi_\th)$ and use standard optimizers like Adam~\citep{kingma2014adam}; ~3) We use $\eta=0.5$ and $T=10$ by default. 
        \ENDFOR
    \end{algorithmic}
    \label{alg:bome}
\end{algorithm*}

\vspace{-5pt}
\section{Analysis}
\vspace{-5pt}
We first elaborate the KKT condition of \eqref{eq:bo_constrain} (Section~\ref{sec:kkt}), then quantify the convergence of the method by how fast it meets the KKT condition. 
We consider both the case when $g$ satisfies the 
 Polyak-{\L}ojasiewicz (PL) inequality w.r.t. $\theta$,  
 hence having a unique global optimum (Section~\ref{sec:pl_theory}), 
 and when $g$ have multiple local minimum (Section~\ref{sec:kl_theory}). 
\vspace{-5pt}
\subsection{KKT Conditions} \label{sec:kkt} 
\vspace{-5pt}
Consider a general constrained optimization of form $\min f(v,\ss)$ s.t. $q(v,\ss)\leq0$.  
Under proper regularity conditions known as constraint quantifications \citep{nw2006numerical}, 
the first-order KKT condition gives a necessary condition for a feasible point $(v^*, \ss^*)$ with $q(v^*,\ss^*)\leq 0$ to be a local optimum of \eqref{eq:bo_constrain}: There exists 
a Lagrangian multiplier  $\lambda^* \in  [0, +\infty)$, such that
\bbb \label{eq:kkt0}  
\begin{split} 
& \dd f(v^*, \ss^*) + \lambda^* \dd q(v^*, \ss^*) =0,  \\
\end{split}
\eee 
and $\lambda^*$  satisfies the complementary slackness condition $\lambda^* q(v^*, \ss^*) =0$. 
A common regularity condition to ensure \eqref{eq:kkt0} is the \emph{constant rank constraint quantification (CRCQ)} condition \citep{janin1984directional}. 
\begin{definition}
A point $(v^*, \theta^*)$  is said to satisfy CRCQ with
a function $h$ if  the rank of the Jacobian matrix $\dd  h(v,\theta)$ is constant in a neighborhood of $(v^*,\theta^*)$. 
\end{definition}

Unfortunately,
\textbf{the KKT condition in \eqref{eq:kkt0} does not hold for the bilevel optimization in \eqref{eq:bo_constrain}.} 
The CRCQ condition does not typically hold for this problem. 
This is because the minimum of $q$ is zero, and hence if 
$(v^*,\theta^*)$ is feasible for \eqref{eq:bo_constrain}, 
then $(v^*,\theta^*)$ must attain the minimum of $q$, yielding $q(v^*,\theta^*)=0$ and $\dd q(v^*, \theta^*)=0$ if $q$ is smooth; 
but we could not have $\dd q(v,\theta) =0$ uniformly in a neighborhood of $(v^*,\theta^*)$ (hence CRCQ fails) unless $q$ is a constant  around $(v^*,\theta^*)$. In addition, 
 if KKT \eqref{eq:kkt0} holds,  
 we would have  $\dd f(v^*, \ss^*) =-\lambda^*\dd q(v^*, \theta^*)= 0$ which happens  only in the rare case when $(v^*,\theta^*)$ is a stationary point of both $f,g$.

Instead, one can establish a KKT condition of BO 
through the form in \eqref{eq:gbo},
{because there is nothing special that prevents $(v^*,\theta^*)$ from satisfying CRCQ with $\dd_\theta q = \dd_\theta g$ (even though we just showed that it is difficult to have CRCQ with $q$)}. 
Assume $f$ and $\dd_\theta q$   are continuously differentiable, and $(v^*, \th^*)$ is a point satisfying $\dd_\theta q(v^*, \theta^*) =0$ and CRCQ with $\dd_\theta q$.  
Then by the typical first order KKT condition of \eqref{eq:gbo}, there exists a 
Lagrange 
multiplier $\omega^* \in \RR^n$ 
such that  
\bbb\label{eq:kkt2}
\nabla f(v^*,\th^*)+\nabla (\dd_\theta q(v^*,\th^*))\omega^*=0.
\eee 
This condition can be viewed as the limit of a sequence of \eqref{eq:kkt0} 
in the following way: 
assume we relax the constraint in \eqref{eq:bo_constrain} to $q(v,\theta)\leq c_k$ 
where $c_k$ is a sequence of positive numbers that converge to zero, then we can establish \eqref{eq:kkt0} for each $c_k>0$ and pass the limit to zero to yield \eqref{eq:kkt2}.

\begin{proposition} \label{pro:kkt}
Assume that $f$, $q$, $\nabla q$ are continuously differentiable and $\norm{\nabla{f}}, f$ is bounded.
For a {feasible} 
point $(v^*, \theta^*)$ of \eqref{eq:bo_constrain} that satisfies CRCQ with {$\dd_\theta q$}, 
if $(v^*, \theta^*)$
is the limit of a sequence $\{(v_k, \theta_k)\}_{k=1}^\infty$ satisfying
 $q(v_{k},\th_{k})\neq0$ $\forall k$,  
 and there exists a sequence  $\{\lambda_k\} \subset [0,\infty)$ such that 
\bb 
\nabla f(v_{k},\th_{k})+\lambda_k \nabla q(v_{k},\th_{k}) \to 0,  &&
q(v_k, \th_k) \to 0, 
\ee 
  as $k\to +\infty$, 
then $(v^*,\th^*)$ satisfies \eqref{eq:kkt2}. 
\end{proposition}
This motivates us to use the following function 
as a measure of stationarity of the solution returned by the algorithm: 
\[
\K(v,\th)=\underset{\text{local\ improvement}}{\underbrace{{\textstyle \min_{\lambda\ge0}}||\nabla f(v,\th)+\lambda\nabla q(v,\th)||^{2}}}+\underset{\text{feasibility}}{\underbrace{q(v,\th)}}.
\]
The hope is to have an algorithm that generates a sequence $\{(v_k, \theta_k)\}_{k=0}^\infty$ that satisfies $\K(v_k, \theta_k) \to 0$ as $k\to+\infty$.

Intuitively, the first term in $\K(v,\th)$ measures how much $\dd f$ conflicts with $\dd q$ (how much we can decrease $f$ without increasing $q$), as it is equal to the squared $\ell_2$ norm of the solution to the problem
$
\min_{\delta}||\nabla f-\delta||^{2}\ s.t.\ \left\langle \nabla q,\delta\right\rangle \ge0.
$
The second term in $\K$ measures how much the $\argmin g$ constraint is satisfied. 
\vspace{-5pt}
\subsection{Convergence with unimodal $g$} \label{sec:pl_theory}
\vspace{-5pt}
We first present the convergence rate 
when assuming $g(v,\cdot)$ has unique minimizer and satisfies the   Polyak-{\L}ojasiewicz (PL) inequality for all $v$, which %
guarantees a linear convergence rate of the gradient descent on the low level problem. 
\vspace{-3pt}
\begin{assumption}[PL-inequality] \label{asm:PL-inequality}
 Given any $v$, assume $g(v,\cdot)$ has a unique minimizer denoted as $\th^*(v)$. Also assume there exists $\kappa>0$ such that for any $(v,\theta)$, $\left\Vert \nabla_{\th}g(v,\th)\right\Vert ^{2}\ge\kappa(g(v,\th)-g(v,\th^*(v)))$.
\end{assumption}
\vspace{-3pt}
The PL inequality %
gives a characterization on how
a small gradient norm implies global optimality. 
It is implied from, but weaker than strongly convexity. 
The PL-inequality is more appealing than convexity because some modern over-parameterized deep neural networks have been shown to satisfy the PL-inequality along the trajectory of gradient descent. 
See, for example, \citet{frei2021proxy,song2021subquadratic,liu2022loss} for more discussion.
\vspace{-3pt}
\begin{assumption}[Smoothness] \label{asm:Smoothness}
$f$ and $g$ are differentiable, and $\dd f$ and $\dd g$ are $L$-Lipschitz w.r.t. the joint inputs $(v,\theta)$ for some $L\in(0,+\infty)$. 
\end{assumption} 
\vspace{-3pt}
\begin{assumption}[Boundedness] \label{asm:Boundedness}
 There exists a constant $M<\infty$ such that 
 $\norm{\nabla g(v,\th)}$, $\norm{\nabla f(v,\th)}$, $|f(v,\th)|$ and $|g(v,\th)|$ are all upper bounded by $M$ 
 for any $(v,\th)$. 
\end{assumption} 
\vspace{-3pt}
Assumptions \ref{asm:Smoothness} and \ref{asm:Boundedness} are  both standard in optimization.
\vspace{-3pt}
\begin{theorem} \label{thm:Convergence k}
Consider Algorithm~\ref{alg:bome} with $\xi, \alpha\le1/L$, $\phi_k = \eta \norm{\nabla \hat{q}(v_k,\th_k)}^2$,  and $\ul>0$. 
Suppose that Assumptions~\ref{asm:PL-inequality}, \ref{asm:Smoothness}, and \ref{asm:Boundedness}  hold. Then there exists a constant $c$ depending on $\alpha,\kappa,\ul ,L$ such that when 
$T \geq c$, we have for any $K \geq 0$, 
\begin{align*}
 \min_{k\leq K}\, \K(v_{k},\th_{k})\!=\!O\!\left ( \sqrt{\xi }  + \sqrt{\frac{q_0}{\xi K}}  + \frac{1}{\xi K}  + \exp(- b T) \right )
\end{align*}
where $q_0 = q(v_{0},\th_{0})$, and $b>0$ is a constant depending on $\kappa$, $L$, and {$\alpha$}.
\end{theorem}

\begin{remark}
Note that one of the dominant terms depends on the initial value $q_0 = q(v_{0},\th_{0})$. Therefore, we can obtain a better rate if we start from a $\th_0$ with small $q_0$ (hence near the optimum of $g(v_0,\cdot)$). In particular, when $q(v_{0},\th_{0})=O(1)$, choosing %
$\xi=O(K^{-1/2})$ gives $ \min_{k\leq K}\K(v_{k},\th_{k}) = O(K^{-1/4} + \exp(-bT))$ rate. 
On the other hand, 
if we start from a better initialization such that {$q(v_{0},\th_{0})=O((\xi K)^{-1})$}, then 
choosing 
$\xi=O(K^{-2/3})$  %
gives $\min_{k\leq K}\K(v_{k},\th_{k}) = O(K^{-1/3}+\exp(-bT))$. %
\end{remark}
\vspace{-5pt}
\subsection{Convergence with multimodal $g$}
\vspace{-5pt}
\label{sec:kl_theory}
The PL-inequality eliminates the possibility of having stationary points that are not global optimum. 
To study  cases in which $g$ has multiple local optima, we introduce the notion of attraction points 
following gradient descent. 
\vspace{-3pt}
\begin{definition} [Attraction points]
Given any $(v,\th)$, we say that  $\th^{\diamond}(v,\th)$ is  the attraction point of $(v,\th)$ with step size $\alpha>0$ 
if the sequence $
\{\theta^{(t)}\}_{t=0}^\infty$ generated by gradient descent $\th^{(t)}=\th^{(t-1)}-\alpha\nabla_{\th}g(v,\th^{(t-1)})$
starting from $\th^{(0)}=\th$ converges to $\th^{\diamond}(v,\th)$.
\end{definition}
\vspace{-3pt}
Assume the step size ${\alpha}\le1/L$ where $L$ is the smoothness constant defined in Assumption \ref{asm:Smoothness}, one can show the existence {and uniqueness} of attraction point of any $(v,\th)$ using Proposition 1.1 of \citet{traonmilin2020basins}. 
Intuitively, the attraction of $(v,\th)$ is where the gradient descent
algorithm can not make improvement. 
In fact, when ${\alpha}\le1/L$, 
one can show that $g(v,\theta) \leq g(v,\theta^{\diamond}(v,\theta))$ is equivalent to the stationary condition  $\dd_\theta g(v,\theta)=0$. 

The set of $(v,\theta)$ that have the same attraction point forms an attraction basin. 
Our analysis needs to assume the PL-inequality within the individual attraction basins. 
\vspace{-3pt}
\begin{assumption}[Local PL-inequality within attraction basins] \label{asm:KL-inequality}
Assume that for any $(v,\th)$, $\theta^\diamond(v,\theta)$ exists. Also assume that there exists $\kappa>0$ such that for any $(v,\th)$ $\left\Vert \nabla_{\th}g(v,\th)\right\Vert ^{2}\ge\kappa(g(v,\th)-g(v,\theta^{\diamond}(v,\theta))$.
\end{assumption}
\vspace{-3pt}

We can also define local variants of $q$ and $\mathcal K$ as follows:
\[
q^\diamond(v,\theta)  = g(v,\th) - g(v, \th^{\diamond}(v,\th)),
~~~
\K^{\diamond}(v,\th)=  \min_{\lambda\ge0}\norm{\nabla f(v,\th)+\lambda\nabla q^{\diamond}(v,\th)}^{2} + {q^{\diamond}(v,\th)}.
\]
Compared with Section~\ref{sec:pl_theory}, a key technical challenge is that
$\theta^\diamond(v,\theta)$ and hence $q^\diamond(v,\theta)$ can be discontinuous w.r.t. $\th$ when it is on the boundary of different attraction basins; %
$\K^{\diamond}$ is not well defined on these points. 
However, these boundary points are not stable stationary points, and it is possible to use arguments based on the stable manifold theorem to show that an algorithm with random initialization will almost surely not visit them \citep{shub2013global,lee2016gradient}. 
\vspace{-3pt}
\begin{theorem} \label{thm:nonconvex_converge} 
Consider Algorithm~\ref{alg:bome} with $\xi, \alpha\le1/L$, $\phi_k = \eta \norm{\nabla \hat{q}(v_k,\th_k)}^2$,  and $\ul>0$. 
Suppose that Assumptions~\ref{asm:Smoothness}, \ref{asm:Boundedness}, and \ref{asm:KL-inequality} hold and that $q^\diamond$ is differentiable on $(v_k,\th_k)$ at every iteration $k\geq 0$.
Then there exists a constant $c$ depending on $\alpha,\kappa,\ul ,L$, such that when $T \geq c$, we have
\[
\min_{k\le K}\, \K^{\diamond}(v_{k},\th_{k})=O\left (
\sqrt{\xi}  %
+ \sqrt{\frac{1}{\xi K}}   
 + \exp(-bT)  \right ), 
\]
where $b$ is a positive constant depending on $\kappa$, $L$, and $\alpha$.
\end{theorem}
\vspace{-3pt}
Unlike Theorem~\ref{thm:Convergence k},  the rate does not improve when $q_0^\diamond := q^{\diamond}(v_0, \th_0)$ is small because the attraction basin may change in different iterations, eliminating the benefit of starting from a good initialization. Choosing %
$\xi=O(K^{-1/2})$ gives $O(K^{-1/4}+\exp(-bT))$ rate of $\min_{k\leq K}\K^{\diamond}(v_{k},\th_{k})$.

\vspace{-5pt}
\section{Related Works} 
\vspace{-5pt}
The value-function formulation \eqref{eq:bo_constrain}  
is a classical approach in bilevel optimization \citep{outrata1990numerical,ye1995optimality,dempe2020bilevel}. 
However, 
despite its attractive properties, 
it has been mostly used as a theoretical tool, 
and much less exploited for practical algorithms compared with the more widely known hypergradient approach (Section~\ref{sec:background}), especially for challenging nonconvex functions $f$ and $g$ such as those encountered in deep learning.  
One exception is \citet{liu2021value}, 
which proposes a BO method by solving the value-function formulation using an interior-point method combined with a smoothed approximation.
This  was improved later in 
a pessimistic trajectory truncation approach \citep{liu2021towards} and a sequential minimization approach \citep{liu2021valueseq} (BVFSM). 
Similar to our approach, these methods do not require computation of Hessians, thanks to the use of value function.
However, as we observe in experiments (Section~\ref{sec:observations}), BVFSM tends to be dominated by our method both in accuracy and speed,  
and is sensitive to some hyperparameters that are difficult to tune (such as the coefficients of the log-barrier function in interior point method). 
Theoretically,
 \citet{liu2021value, liu2021valueseq, liu2021towards}  provide only asymptotic analysis on the convergence of the smoothed and penalized surrogate loss to the target loss. 
 They do not give an analysis for the algorithm that was actually implemented.
 
Our algorithm is build up on the dynamic control barrier method of \citet{gong2021automatic}, an elementary approach for constrained optimization. 
\citet{gong2021automatic} also applied their approach to solve a 
lexicographical optimization of form $\min_{\theta} f(\theta)$ s.t. $\theta\in \argmin_{\theta'} g(\theta')$, which is a bilevel optimization without an outer variable (known as \emph{simple bilevel optimization}  \citep{dempe2021simple}).  Our method is an extension of their method to general bilevel optimization. {Such extension is not straightforward, especially when the lower level problem is non-convex, requiring introducing the stop-gradient operation in a mathematically correct way.} We also provide non-asymptotic analysis for our method, that goes beyond the continuous time analysis in~\citet{gong2021automatic}. A key sophistication in the theoretical analysis is that we need to control the approximation error of $\theta^*(v_k)$  with $\theta_k^{(T)}$ at each step, which requires an analysis significantly different from that of \citet{gong2021automatic}. {Indeed, non-asymptotic results have not yet been obtained for many BO algorithms. Even for the classic hypergradient-based approach (such results are established only recently in \citet{ji2021bilevel}). We believe that we are the first to establish a non-asymptotic rate for a purely first-order BO algorithm under general assumptions, e.g. the lower level problem can be both convex or non-convex.}
 
Another recent body of 
theoretical works on BO focus on how to optimize when only stochastic approximation of the objectives is provided   \citep{ghadimi2018approximation,hong2020two,ji2021bilevel,yang2021provably,guo2021randomized,chen2021single,khanduri2021near}; there are also recent works  on the lower bounds and minimax optimal algorithms \citep{ji2021lower,ji2021bilevel}. These algorithms and analysis are based on  hypergradient descent and hence require Hessian-vector products in  implementation.

\vspace{-5pt}
\section{Experiment}
\vspace{-5pt}
\label{sec:experiment}
We conduct experiments (1) to study the correctness, basic properties, and robustness to hyperparameters of BOME, and (2) to test its performance and computational efficiency on challenging ML applications, compared with state-of-the-art bilevel algorithms. In the following, we first list the baseline methods and how we set the hyperparameters. Then we introduce the experiment problems in Section~\ref{sec:problems}, which includes 3 toy problems and 3 ML applications, and provide the experiment results. Finally we summarize observations and findings in Section~\ref{sec:observations}.

\textbf{Baselines} A comprehensive set of state-of-the-art BO methods are chosen as baseline methods. This includes the \emph{fully first-order} methods: BSG-1~\citep{giovannelli2021bilevel} and  BVFSM~\citep{liu2021valueseq},
; %
a \emph{stationary-seeking} method: Penalty ~\citep{mehra2021penalty},
\emph{explicit/implicit} methods: ITD~\citep{ji2021bilevel}, AID-CG (using conjugate gradient), AID-FP (using fixed point method)~\citep{grazzi2020iteration}, reverse (using reverse auto-differentiation)~\citep{franceschi2017forward} stocBiO~\citep{ji2021bilevel}, and VRBO~\citep{yang2021provably}.

\textbf{Hyperparameters}
Unless otherwise specified, BOME strictly follows Algorithm~\ref{alg:bome} with 
$\phi_k =\eta \norm{\nabla \hat{q}(v_k, \th_k)}^2$, $\eta=0.5$, and $T = 10$. The inner stepsize $\alpha$ is set to be the same as outer stepsize $\xi$. The stepsizes of all methods are set by a grid search from the set $\{0.01, 0.05, 0.1, 0.5, 1, 5, 10, 50, 100, 500, 1000\}$. All toy problems adopt vanilla gradient descent (GD) and applications on hyperparameter optimization adapts GD with a momentum of $0.9$. Details are provided in Appendix~\ref{sec:apx-exp}.

\vspace{-5pt}
\subsection{Experiment Problems and Results}
\vspace{-5pt}
\label{sec:problems}

\textbf{Toy Coreset Problem} 
To validate the \emph{convergence} property of BOME, we consider:
\begin{equation*}
    \begin{split}
    \textstyle{\min}_{v, \theta} \norm{\theta  - x_0}^2 
    ~~~s.t.~~~ \theta \in \textstyle{\argmin}_\theta \norm{\theta - X\sigma(v)}^2,
    \end{split}
    \label{eq:toy-coreset}
\end{equation*}
where  $\sigma(v) = \exp(v)/\sum_{i=1}^4 \exp(v_i)$
is the softmax function, $v \in \mathbb{R}^4, \theta\in \mathbb{R}^2$, and $X= [x_1, x_2, x_3, x_4] \in \mathbb{R}^{2\times 4}$. 
The goal is to find the closest point  to a target point $x_0$ 
within the convex hull of $\{x_1,\ldots, x_4\}$.
See Fig.~\ref{fig:toy_convergence_and_minimax} (upper row) for the illustration and results.
\begin{figure*}[t!]
    \centering
    \vspace{-20pt}
    \label{fig:toy_convergence_and_minimax}
    \includegraphics[width=\textwidth]{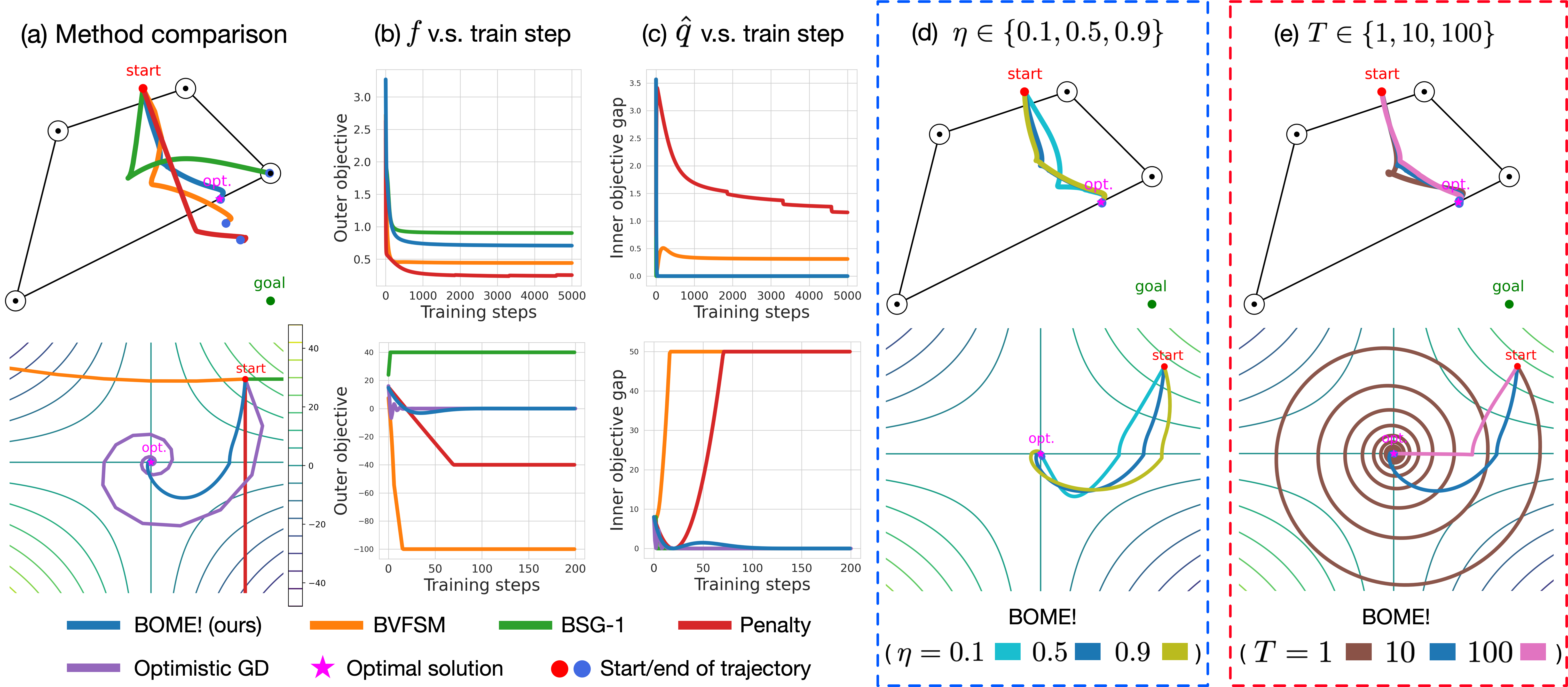}
    \vspace{-10pt}
    \caption{Results on the toy coreset  problem  and mini-max  problem. (a)-(c): the trajectories of $(v_k,\theta_k)$, and $f(v_k,v_k)$ and $\hat q_k(v_k,v_k)$ of BOME (our method), BSG-1~\citep{giovannelli2021bilevel}, BVFSM~\citep{liu2021valueseq}, Penalty ~\citep{mehra2021penalty} and Optimistic GD~\citep{daskalakis2017training} (only for minimax problem).  
    (d)-(e) trajectories of BOME with different choices of inner gradient step $T$ and the control coefficient $\eta$.
    }
    \vspace{-0.6cm}
\end{figure*}

\textbf{Toy Mini-Max Game} 
Mini-max game is a special and challenging case of BO where $f$ and $g$ contradicts with each other completely (e.g., $f=-g$). We consider
\begin{equation}
    \textstyle{\min}_{v,\theta \in \mathbb{R}}~v\th ~~~~s.t.~~~~ \th \in \textstyle{\argmax}_{\th' \in \mathbb{R}}~v\th'.
    \label{eq:toy-minimax}
\end{equation}
The optimal solution is $v^* = \th^* = 0$.  
Note that the naive gradient descent ascent algorithm diverges to infinity on this problem, 
and a standard alternative is to use optimistic gradient descent~\citep{daskalakis2017training}. 
Figure~\ref{fig:toy_convergence_and_minimax} (lower row) shows that BOME works on this problem while other first-order BO methods fail. 

\textbf{Degenerate Low Level Problem}
Many existing BO algorithms require the low level singleton (LLS) 
assumption, which BOME does not require. 
To test this, we consider an example from \citet{liu2020generic}: 
\begin{equation*}
    \begin{split}
     \textstyle{\min}_{v\in \RR, \theta\in\RR^2}
     \norm{\theta - [v;1]}_2^2 
    ~~~~ s.t.~~~~ \th \in \textstyle{\argmin}_{(\theta_1',\theta_2')\in \RR^2}  (\theta_1' - v)^2,    
    \end{split}
    \label{eq:toy-lls}
\end{equation*}
where $\th = (\th_1, \th_2) $ and the solution is $v^*=1, \th^*=(1,1)$. 
See Fig.~\ref{fig:toy_lls_full} in Appendix A.3 for the result.

\textbf{Data Hyper-cleaning}  
We are given a noisy training set $\mathcal D_{\text{train}}:= \{x_i, y_i\}_{i=1}^m$ and a clean validation set  $\mathcal D_{\text{val}}$. 
 The goal is to optimally weight the training data points so that the model trained on the weighted training set yields good performance on the validation set: 
 \begin{align*} 
    \textstyle{\min}_{v, \theta} 
    \ell^{\text{val}}(\theta), ~~~ 
    s.t.~~~\theta = \textstyle{\argmin}_{\theta'}
    \left\{ \ell^{\text{train}}(\theta', v)
    + c \norm{\theta'}^2 \right\},
\end{align*} 
where $\ell^{\text{val}}$ is the validation loss on $\mathcal D^{\text{val}}$, and $\ell^{\text{train}}$ is a weighted training loss: $\ell^{\text{train}} = \sum_{i=1}^m \sigma(v_i)  \ell(x_i, y_i, \theta)$ with $\sigma(v)= \text{Clip}(v, [0,1])$ and 
$v \in \RR^m$. 

We set $c = 0.001$. For the dataset, we use MNIST~\citep{deng2012mnist} (FashionMNIST~\citep{xiao2017fashion}).
We corrupt 50\% of the training points by assigning them randomly sampled labels. 
See Fig.~\ref{fig:ho} (upper panel) for the results. (Results for FashionMNIST are reported in Appendix~\ref{sec:apx-exp-hyperclean}.)

\begin{figure*}[t]
    \centering
    \vspace{-0.8cm}
    \includegraphics[width=\textwidth]{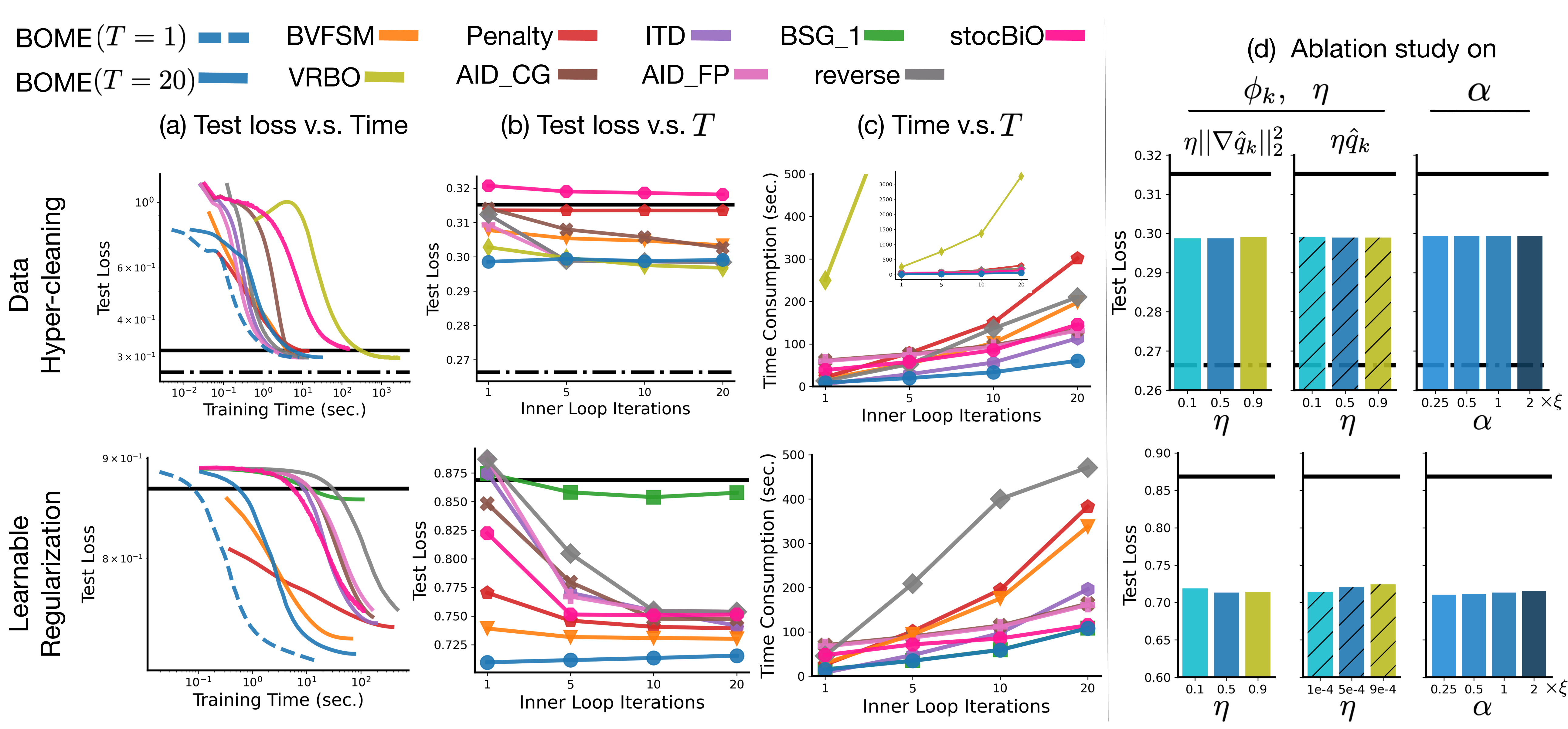}
    \vspace{-18pt}
    \caption{Result for hyperparameter optimization. \textbf{Top:} data hyper-cleaning on MNIST dataset. The solid black line is the model performance trained purely on the validation set and the dashed black line is the model performance trained on the validation set and on the part of training set that have correct labels. \textbf{Bottom:} learnable regularization on 20 Newsgroup dataset. The solid black line indicates the model performance without any regularization. All results are averaged on 5 random trials. See Appendix~\ref{sec:apx-exp-hyperclean} and~\ref{sec:apx-exp-regularization} for results on FashionMNIST and more details.}
    \vspace{-1pt}
    \label{fig:ho}
\end{figure*}

\textbf{Learnable Regularization}
We apply bilevel optimization to learn the optimal regularization coefficient on the twenty newsgroup dataset:\footnote{Dataset from \url{https://scikit-learn.org/0.19/datasets/twenty_newsgroups.html}.}
\begin{align*} 
\vspace{-0.3cm}
    \textstyle{\min}_{v,\theta}\,  
   \ell^{\text{val}}(\theta)
    ~~~\text{s.t.}~~~\theta \in \textstyle{\argmin}_{\theta'} \left\{ \ell^{\text{train}} (\theta') + \norm{W_v \theta'}^2_2 \right\},
\vspace{-0.3cm}
\end{align*}
where $W_v$ is a matrix depending on $v$, e.g., $W_v = \diag(\exp(v))$. %
See Fig.~\ref{fig:ho} (lower panel) for results. 

\textbf{Continual Learning (CL)} 
CL 
studies how to learn on a sequence of tasks in an online fashion without catastrophic forgetting of previously learned tasks. 
We follow the setting of contextual transformation network (CTN) from \citet{pham2020contextual}, which trains a deep neural network consisting of
a quickly updated  backbone network (parameterized by $\theta$) and 
a slowly updated controller network  (parameterized by $v$). %
When training the $\tau$-th task, we update $(v, \theta)$ by %
$$
    \textstyle{\min}_{v, \theta} ~\ell_{1:\tau}^\text{val}\big(v, \th\big) ~~~~~s.t.~~~~~\th \in \textstyle{\argmin}_{\th'}  ~\ell_{1:\tau}^\text{train}\big(v, \th'\big),
$$
where $\ell_{1:\tau}^\text{train}$ and $\ell_{1:\tau}^\text{val}$ are
 the training and validation loss available up to task $\tau$.  %
 The goal is to update the controller such that the long term loss $\ell_{1:\tau}^\text{val}$ is minimized assuming $\theta$ 
is adapted to the available training loss when new tasks come.
\newcommand{\tend}{t}%
Assume the CL process terminates at time $\tend$ . 
Denote by $a^s_\tau$ 
the test accuracy of task $s$ after training on task $\tau$. 
We measure the performance of CL by  1) the final mean accuracy on all seen tasks ($\text{ACC} = \frac{1}{\tend}\sum_{\tau\leq \tend}  a^\tau_{\tend}$), 2) how much the model forgets as measured by negative backward transfer $\text{NBT} = \frac{1}{\tend}\sum_{\tau \leq \tend} (a^\tau_\tau - a^\tau_\tend)$, 
  and 3) how fast the model learns on new tasks as measured by forward transfer 
  $\text{FT} = \frac{1}{\tend}\sum_{\tau\leq \tend} a^\tau_\tend$. Note that $\text{FT} = \text{ACC} + \text{NBT}$.

We follow the setting of \citet{pham2020contextual} closely, 
except replacing their bilevel optimizer (which is essentially ITD~\citep{ji2021bilevel}) with BOME. See Appendix~\ref{sec:apx-exp-cl} for experiment details. 
The results are shown in Table~\ref{tab:cl},
where 
in addition to the bilevel algorithms,
we also compare with a set of state-of-the-art CL algorithms, including MER~\citep{riemer2018learning}, ER~\citep{chaudhry2019tiny}, and GEM~\citep{lopez2017gradient}. 
Table~\ref{tab:cl} also 
includes an `Offline" basline -- learning $t$ tasks simultaneously using a single model (which is the upper bound on performance).

\begin{table*}[t]
\centering
\vspace{-0.1cm}
\resizebox{\textwidth}{!}{
\begin{tabular}{lcccccc}
\toprule 
\multirow{3}{*}{Method} & \multicolumn{3}{c}{PMNIST} & \multicolumn{3}{c}{Split CIFAR}\\
\cmidrule(lr){2-4}  \cmidrule(lr){5-7}
 & ACC~$(\uparrow)$ & NBT~$(\downarrow)$ & FT~$(\uparrow)$ & ACC~$(\uparrow)$ & NBT~$(\downarrow)$ & FT~$(\uparrow)$\tabularnewline
\midrule
Offline & $84.95\pm0.95$ & - & - & $74.11\pm0.66$ & - & -\tabularnewline
[0.05cm]
MER & $76.59\pm0.74$ & $5.73\pm0.59$ & $82.32\pm0.34$ & $60.32\pm0.86$ & $8.91\pm0.86$ & $69.23\pm0.40$\tabularnewline
[0.05cm]
CTN~(+ITD) & $78.40\pm0.28$ & $5.62\pm0.39$ & $84.02\pm0.29$ & $67.7\pm60.96$ & $4.88\pm0.77$ & $72.58\pm0.62$\tabularnewline
[0.05cm]
CTN~(+BVFSM) & $77.78\pm0.32$ & $7.25\pm0.28$ & $\pmb{85.03}\pm0.28$ & $67.04\pm0.76$ & $6.97\pm0.62$ & $\pmb{74.01}\pm0.57$ \tabularnewline
[0.05cm]
CTN~(+BOME) & $\pmb{80.70}\pm0.26$ & $\pmb{4.09}\pm0.27$ & $\pmb{84.79}\pm0.25$ & $\pmb{68.16}\pm0.60$ & $\pmb{4.72}\pm0.75$ & $72.88\pm0.48$\tabularnewline
\bottomrule 
\end{tabular}
}
\vspace{-5pt}
\caption{Results of continual learning as bilevel optimization. We compute the mean and standard error of each method's results over 5 independent runs.
Best results are \textbf{bolded}. The full result with comparison against other methods are provided in Table~\ref{tab:cl-full} in the Appendix.
}
\vspace{-0.65cm}
\label{tab:cl}
\end{table*}

\vspace{-8pt}
\subsection{Observations}
\vspace{-5pt}
\label{sec:observations}
\textbf{BOME yields faster learning and better solutions at convergence}
Figure~\ref{fig:toy_convergence_and_minimax}-\ref{fig:toy_lls_full} show that BOME converges to the optimum of the corresponding bilevel problems and work well on the mini-max optimization and the degenerate low level problem; see also Fig.\ref{fig:toy_convergence} in Appendix~\ref{sec:apx-exp-coreset}. 
In comparison, the other methods like BSG-1, BVFSM, and Penalty 
fail to converge to the true optimum even with a grid search over their hyperparameters.
Moreover, in all three toy examples, BOME guarantees that $\hat{q}$, which is a proxy for the optimality of the inner problem, decreases to $0$. 
From Fig.~\ref{fig:ho}, it is observed that BOME achieves comparable or better performance than the state-of-the-art bilevel methods for hyperparameter optimization. Moreover, BOME exhibits better computational efficiency (Fig.~\ref{fig:ho}), especially on the twenty newsgroup dataset where the dimension of $\theta$ is large. 
In Table~\ref{tab:cl}, we find that directly plugging in BOME to the CL problem yields a substantial performance boost. 
\vspace{-2pt}

\textbf{Robustness to parameter choices} 
Besides the standard step size $\xi$ in typical optimizers, 
BOME only has three parameters:  control coefficient $\eta$, inner loop iteration $T$, and inner step size $\alpha$. 
We use the default setting of $\eta = 0.5$, $T = 10$ and $\alpha = \xi$ 
across the experiments.
From Fig.~\ref{fig:toy_convergence_and_minimax} (d,e) and Fig.~\ref{fig:ho} (b,d), BOME is robust to the choice of  $\eta$, $T$ and $\alpha$ as varying them results in almost identical performance. Specifically, $T=1$ works well in many cases (see Figure \ref{fig:toy_convergence_and_minimax} (e) and \ref{fig:ho} (b)). The fact that BOME works well with a small $T$ empirically makes it computationally attractive in practice.
\vspace{-4pt}

\textbf{Choice of control barrier $\phi_k$}
The control barrier is set as $\phi_k = \eta\norm{\nabla \hat{q}(v_k,\theta_k)}^2$  by default. Another option is to use $\phi_k = \eta\hat{q}(v_k,\theta_k)$. We test both options on the data hyper-cleaning and learnable regularization experiments in Fig.~\ref{fig:ho} (d), and observe no significant difference (we choose $\eta$ properly so that both choices of $\phi_k$ is on the same order). Hence we use $\phi_k = \eta\norm{\nabla \hat{q}(v_k,\theta_k)}^2$ as the default. 
\vspace{-2pt}

\textbf{Comparison against BVFSM} The most relevant baseline to BOME is BVFSM, which similarly adopts the value-function reformulation of the bilevel problems. However, BOME consistently outperform BVFSM in both converged results and computational efficiency, across all experiments. 
More importantly, BOME has fewer hyperparameters and is robust to them,
while we found BVFSM is sensitive to hyperparameters. This makes BOME a better fit for large practical bilevel problems.

\vspace{-5pt}
\section{Conclusion and Future Work}
\vspace{-5pt}
BOME, a simple fully first-order bilevel method, is proposed in this work with non-asymptotic convergence guarantee. 
While the current theory requires the inner loop iterations to scale in a logarithmic order w.r.t to the outer loop iterations, we do not observe this empirically. A further study to understand the mechanism is an interesting future direction.

\bibliography{z_ref}
\bibliographystyle{plainnat}

\clearpage
\section*{Societal Impacts}
This paper proposes a simple first order algorithm for bi-level optimization. Many specific instantiation of bi-level optimization such as adversarial learning and data attacking might be harmful to machine learning system in real world, as a general optimization algorithm for bi-level optimization, our method can be a tool in such process. We also develop a great amount of theoretical works and to our best knowledge, we do not observe any significant negative societal impact of our theoretical result. 

\appendix
\onecolumn

\section{Experiment Details}
\label{sec:apx-exp}
We provide details about each experiment in this section. Regarding the implementation of baseline methods:
\begin{itemize}
    \item BVFSM's implementation is adapted from \url{https://github.com/vis-opt-group/BVFSM}.
    \item Penalty's implementation is adapted from \url{https://github.com/jihunhamm/bilevel-penalty}.
    \item VRBO's implementation is adapted from \url{https://github.com/JunjieYang97/MRVRBO}.
    \item AID-CG and AID-FP implementations are adapted from \url{https://github.com/prolearner/hypertorch}.
    \item ITD implementation is adapted from \url{https://github.com/JunjieYang97/stocBiO}.
\end{itemize}

\subsection{Toy Coreset Problem}
\label{sec:apx-exp-coreset}

The problem is:
\begin{equation*} 
    \begin{split}
    \min_{v, \theta} \norm{\theta  - x_0}^2 
    ~~~~s.t.~~~~ \theta \in \argmin_\theta \norm{\theta - X\sigma(v)}^2,
    \end{split}
    \label{eq:apx-toy-coreset}
\end{equation*}
where  $\sigma(v) = \exp(v)/\sum_{i=1}^4 \exp(v_i)$
is the softmax function, $v \in \mathbb{R}^4, \theta\in \mathbb{R}^2$, and $X= [x_1, x_2, x_3, x_4] \in \mathbb{R}^{2\times 4}$. 
where  $\sigma(v) = \exp(v)/\sum_{i=1}^4 \exp(v_i)$
is the softmax function. Here the outer objective $f$ pushes $\theta$ to towards $x_0$ while the inner objective $g$ ensures $\theta$ remains in the convex hull formed by 4 points in the 2D plane (e.g. $X = [x_1, x_2, x_3, x_4] \in \mathbb{R}^{2\times 4}$). We choose $x_0 = (3, -2)$ and the four points $x_1 = (1,3)$, $x_2 = (3, 1)$, $x_3 = (-2, 2)$ and $x_4 = (-3, 2)$. We set $v_0 = (0, 0, 0, 0)$ and $\th_0 \in \{(0, 3), (-3, 1), (3.5, 1)\}$.  For all methods, we fix both the inner stepsize $\alpha$ and the outer stepsize $\xi$ to be $0.05$ and set $T=10$. For BVFSM and Penalty, we grid search the best hyperparameters from $\{0.001, 0.01, 0.1\}$. For BOME, we choose $\phi = \eta \norm{\nabla \hat{q}}^2$ and ablate over $\eta \in \{0.1, 0.5, 0.9\}$ and $T \in \{1, 10, 100\}$. The visualization of the optimization trajectories over the 3 initial points are plotted in Fig.~\ref{fig:toy_convergence}. 

\begin{figure*}[h!]
    \centering
    \includegraphics[width=\textwidth]{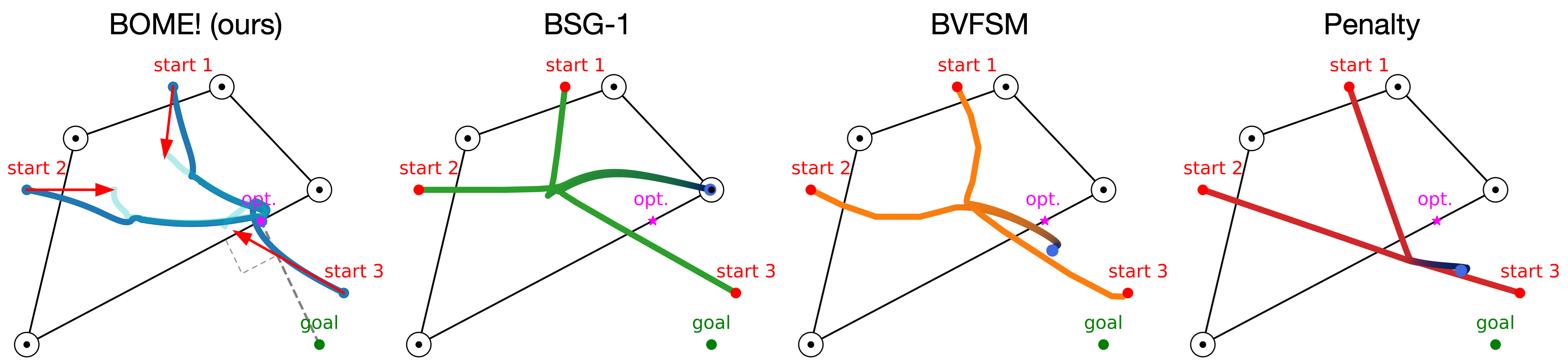}
    \vspace{-12pt}
    \caption{Trajectories of $(v_k, \theta_k)$ on the toy coreset problem \eqref{eq:toy-coreset} obtained from BOME (\textcolor{BOME}{blue}) and three recent first-order bilevel methods: BSG-1 ~\citep{giovannelli2021bilevel} (\textcolor{BSG1}{green}), BVFSM~\citep{liu2021valueseq} (\textcolor{BVFSM}{orange}), and Penalty ~\citep{mehra2021penalty} (\textcolor{Penalty}{red}). 
    The goal of the problem is to find the closet point (marked by \textcolor{magenta}{opt.})
    to the  \textcolor{BSG1}{goal} $x_0$ within the convex envelop of the four vertexes. %
    All methods start from 3 initial points
    (\textcolor{red}{start 1-3}), and the converged points are shown in \textcolor{BlueViolet}{darkblue}.
    For BOME, we also plot the trajectory of $\{\hat{\theta}_k^T\}$ in \textcolor{cyan}{cyan}.}
    \label{fig:toy_convergence}
\end{figure*}
As shown, BOME successfully converges to the optimal solution regardless of the initial $\theta_0$, while BSG-1, BVFSM and Penalty methods converge to non-optimal points. We emphasize that for BVFSM and Penalty, the convergence point \emph{depends on} the choice of hyperparameters.

\subsection{Toy Mini-max Game} 
\label{sec:apx-exp-minimax}
The toy mini-max game we consider is:
\begin{equation}
    \min_{v \in \mathbb{R}}~v\th^*(v) ~~~~s.t.~~~~ \th^*(v) = \argmax_{\th \in \mathbb{R}}~v\th.
    \label{eq:toy-minimax}
\end{equation}
For BOME and BSG-1, BVFSM, and Penalty methods, we again set both the inner stepsize $\alpha$ and $\beta$ to be $0.05$, as no significant difference is observed by varying the stepsizes. For all methods, we set the inner iteration $T=10$. For BVFSM and Penalty, we grid search the best hyperparameters from $\{0.001, 0.01, 0.1\}$.

\subsection{Without LLS assumption} 
\label{sec:apx-exp-lls}
The toy example to validate whether BOME requires the low-level singleton assumption is borrowed from \citet{liu2020generic}:
\begin{equation*}
    \begin{split}
     \min_{v\in \RR, \theta\in\RR^2}
     \norm{\theta - [v;1]}_2^2 
    ~~~~ s.t.~~~~ \th \in \argmin_{(\theta_1',\theta_2')\in \RR^2}  (\theta_1' - v)^2,    
    \end{split}
    \label{eq:toy-lls}
\end{equation*}
where $\th = (\th_1, \th_2) $ and the optimal solution is $v^*=1, \th^*=(1,1)$. Note that the inner objective has infinite many optimal solution $\theta^*(v)$ since it is degenerated. We set both the inner and outer stepsizes to $0.5$ and $T=10$ for all methods. For BVFSM and Penalty, we grid search the best hyperparameters from $\{0.001, 0.01, 0.1\}$. In Fig.~\ref{fig:toy_lls_full}, we provide the distance of $f(v_k, \theta_k), g(v_k, \theta_k), \theta_k, \v_k$ to their corresponding optimal over training time in seconds. Note that BOME ensures that $\hat{q}(v_k, \theta_k) = g(v_k, \th_k) - g(v^*, \th^*)$ decreases to 0.

\begin{figure*}[h!]
    \centering
    \includegraphics[width=\columnwidth]{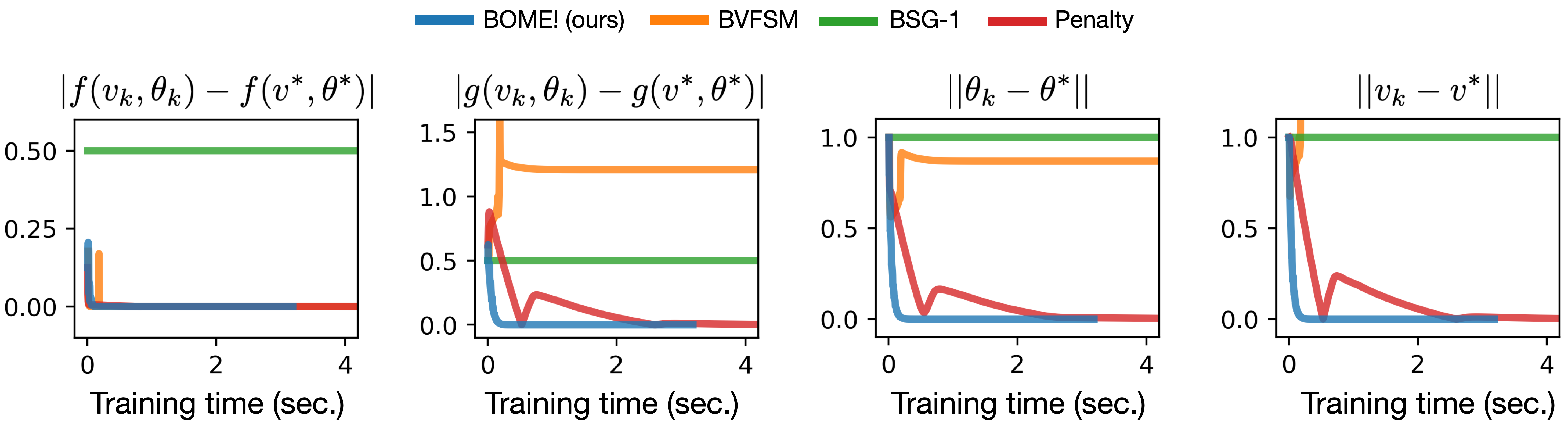}
    \vspace{-15pt}
    \caption{Results on Problem~\eqref{eq:toy-lls} which violates the low-level singleton (LLS). We compare BOME against BSG-1, BVFSM, and Penalty. 
    $(v^*,\theta^*)$ denotes the true optimum and The four plots show how fast  $f(v_k,\theta_k)$, $g(v_k,\theta_k)$, $\theta_k$ and $v_k$ to the corresponding optimal values w.r.t. the training time in seconds.
    }
    \label{fig:toy_lls_full}
\end{figure*}

\subsection{Data Hyper-cleaning}  
\label{sec:apx-exp-hyperclean}
The bilevel problem for data hyper-cleaning is
 \[
    \min_{v, \theta} 
    \ell^{\text{val}}(\theta), ~~~ 
    \text{s.t.}~~~\theta = \argmin_\theta 
    \ell^{\text{train}}(\theta, v)
    + c \norm{\theta}^2,
\]
where $\ell^{\text{val}}$ is the validation loss on $\mathcal D^{\text{val}}$, and $\ell^{\text{train}}$ is a weighted training loss: $\ell^{\text{train}} = \sum_{i=1}^m \sigma(v_i)  \ell(x_i, y_i, \theta)$ with $\sigma(v)= \text{Clip}(v, [0,1])$ and $v \in \mathbb{R}^m$. The training data is of size $m=50000$ and hence $w \in [0,1]^{50000}$. The validation data is of size $m=5000$. The model $\theta = (W, b)$ is a linear model with weight $W$ and bias $b$. Where $W \in \mathbb{R}^{10\times 784}$ and $b \in \mathbb{R}^{10}$.
For this problem, we set inner stepsize $\alpha = 0.01$ for both MNIST and FashionMNIST dataset for all methods as larger or smaller $\alpha$ result in worse performance. As we observe $v$'s gradient norm is much smaller than $\theta$'s in practice, we conduct a grid search over $\xi_v$ from $\{10.0, 50.0, 100.0, 500.0, 1000.0\}$ and also search whether to apply momentum for gradient descent, for all methods. The momentum is searched from $\{0.0, 0.9\}$. For BVFSM and Penalty methods, we also search for their best hyperparameters from $\{0.001, 0.01, 0.1, 1\}$. The model's initial parameter $\theta_0$ is initialized from a pretrained model learned only on the corrupted data. We split the dataset into $4$ parts: train set, validation set 1, validation set 2, and the test set. For each method, the model is learned on the train set, and the hyperparameter $v$ is tuned using validation set 1. The best hyperparameter of any algorithm (e.g. stepsize, barrier coefficient, etc.) are then chosen based on the best validation performance on validation set 2. Then we report the final performance of the model on the test set. To conduct the ablation on $\alpha$ for BOME, we search for $\alpha \in \{0.25\xi, 0.5 \xi, \xi, 2\xi\}$, where $\xi=0.01$ is the best stepsize we found for BOME. Results on MNIST and FashionMNIST dataset are provided in Fig.~\ref{fig:ho_full}. In the first column of Fig.~\ref{fig:ho_full}, we compare BOME with ($T=1$ and $T=20$) with baseline methods whose $T$ is chosen based on best performance on the validation set 2.

\textbf{Remark:} In Fig.~\ref{fig:ho} (top row), we do not include the performance of BSG-1 as we fail to find a set of hyperparameters for BSG-1 to make it work on these data hyper-cleaning problems. VRBO's performance at convergence is tuned by hyparameter search. However, we observe that VRBO learns slowly in practice, as it requires multiple steps of Hessian vector products at each step. We notice that this is slightly inconsistent with the findings in the original paper~\citep{yang2021provably}. We adapt the code from ~\url{https://github.com/JunjieYang97/MRVRBO} and find the original implementation is also slow. It is possible that a good set of hyperparameters can result in better performance.

\begin{figure*}[t]
    \centering
    \includegraphics[width=\textwidth]{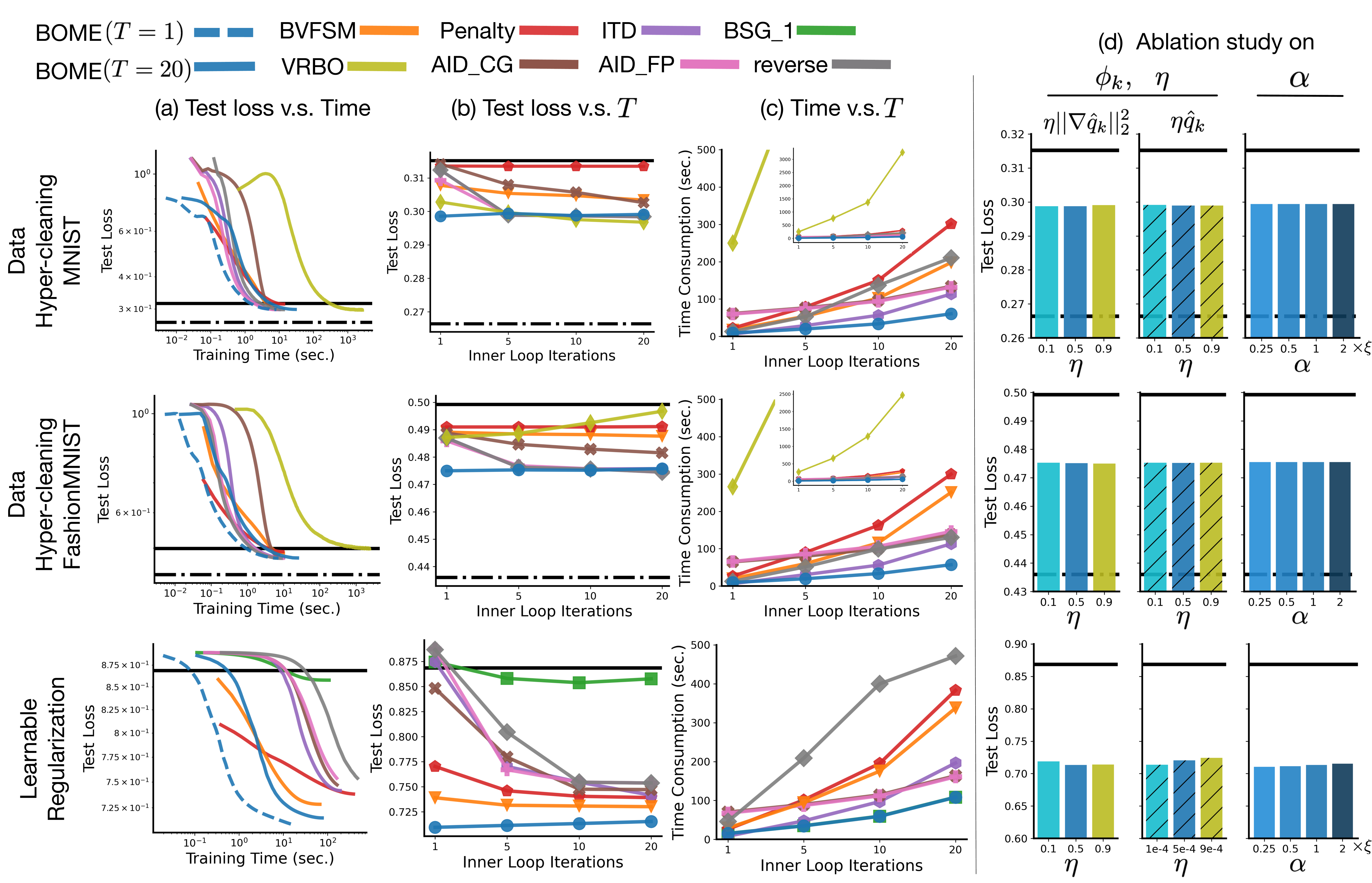}
    \vspace{-15pt}
    \caption{Bilevel optimization for hyperparameter optimization. \textbf{Top:} data hyper-cleaning on MNIST dataset. The solid black line is the model performance trained purely on the validation set and the dashed black line is the model performance trained on the validation set and on the part of training set that have correct labels.
    \textbf{Middle:} data hyper-cleaning on FashionMNIST dataset. \textbf{Bottom:} Learnable regularization on 20 Newsgroup dataset. The solid black line indicates the model performance without any regularization. The results for each method is averaged over 5 independent runs.}
    \label{fig:ho_full}
\end{figure*}

\subsection{Learnable regularization}
\label{sec:apx-exp-regularization}
The bilevel optimization formulation of the learnable rgularization problem is:
\begin{align*} 
    \min_{v,\theta}  
   \ell^{\text{val}}(\theta)
    ~~~~s.t.~~~~\theta \in \argmin_{\theta'}  \ell^{\text{train}} (\theta') + \norm{W_v \theta'}^2_2.
\end{align*}
We use a linear model who's parameter $\theta$ is a matrix 
(e.g. $\theta \in \mathbb{R}^{20\times 130107}$). Hence $v \in \mathbb{R}^{130107}$.
For this experiment, the inner stepsize $\alpha$ of all methods are searched from $\{1, 10, 100, 1000\}$. The outer stepsize $\xi$ is searched from $\{0.5, 1, 5, 10, 50, 100, 500, 1000\}$. For BVFSM and Penalty methods, we also search for their best hyperparameters from $\{0.001, 0.01, 0.1, 1\}$. Similar to the experiment on Data Hyper-cleaning, we split the dataset into $4$ parts: train set, validation set 1, validation set 2, and the test set. The initial model parameter $\theta_0$ is initialized from a pretrained model without any regularization (e.g. $v = 0$) to speed up the learning. To conduct the ablation on $\alpha$ for BOME, we search for $\alpha \in \{0.25\xi, 0.5 \xi, \xi, 2\xi\}$, where $\xi=100$ is the best stepsize we found for BOME.  In the bottom left of Fig.~\ref{fig:ho_full}, we compare BOME with ($T=1$ and $T=20$) with baseline methods whose $T$ is chosen based on best performance on the validation set 2.

\textbf{Remark:} In Fig.~\ref{fig:ho} (bottom row), we do not include the performance of VRBO as we fail to find a set of hyperparameters for VRBO that works well on the learnable regularization experiment.

\subsection{Continual Learning} 
\label{sec:apx-exp-cl}
Continual learning (CL) experiment follows closely to the setup in contextual transformation network (CTN) from \citet{pham2020contextual}, which trains a deep neural network consisting of
a quickly updated  backbone network (parameterized by $\theta$)
a slowly updated controller network  (parameterized by $\theta$).
When training the $\tau$-th task, the update on $(v, \theta)$ is solved from a bilevel optimization: 
\begin{equation*}
    \begin{split}
        \min_{v, \theta} \ell_{1:\tau}^\text{val}\big(v, \th\big) ~~~~\text{s.t.}~~~~~\th \in \argmin_{\th'}  \ell_{1:\tau}^\text{train}\big(v, \th'\big).
    \end{split}
\end{equation*} 

More specifically,
\begin{equation}
    \ell_{1:\tau}^\text{val}\big(v, \th\big) =  L^\text{ctrl}\big(\{\th^*(v), v\}; \mathcal{M}^\text{sm}_{<t+1}\big),~~~~\text{and}~~~~
        \ell_{1:\tau}^\text{train}\big(v, \th'\big) = L^\text{tr}\big(\{\th, v,\}, D_t\cup \mathcal{M}^\text{em}_{<t}\big).
\end{equation}
Here, $\mathcal{M}^\text{sm}_t$ and $\mathcal{M}^\text{em}_t$ denote the semantic and episodic memory of task $t$ (e.g. they can be think of validation and training data) and $D_t$ is the training data of task $t$. Hence, the inner objective learns a backbone $\theta^*(v)$ that performs well on the training data which consists of the current task data $D_t$ as well as previous episodic memories $\mathcal{M}^\text{em}_{<t}$. Then, the outer objective encourages good generalization on the held out validation data, which consists of the semantic memory $\mathcal{M}^\text{sm}_{<t+1}$.
All hyperparameters of BOME are set to the same as those of CTN. We choose $\phi_k = \eta \norm{\nabla \hat{q}_k(v_k, \theta_k}$ where $\eta = 2.0$, here $\eta$ is chosen from $\{0.1, 0.5, 1.0, 2.0\}$.

\begin{table*}[h]
\centering
\resizebox{\textwidth}{!}{
\begin{tabular}{lcccccc}
\toprule 
\multirow{3}{*}{Method} & \multicolumn{3}{c}{PMNIST} & \multicolumn{3}{c}{Split CIFAR}\\
\cmidrule(lr){2-4}  \cmidrule(lr){5-7}
 & ACC~$(\uparrow)$ & NBT~$(\downarrow)$ & FT~$(\uparrow)$ & ACC~$(\uparrow)$ & NBT~$(\downarrow)$ & FT~$(\uparrow)$\tabularnewline
\midrule
Offline & $84.95\pm0.95$ & - & - & $74.11\pm0.66$ & - & -\tabularnewline
[0.05cm]
MER & $76.59\pm0.74$ & $6.88\pm0.59$ & $82.32\pm0.34$ & $60.32\pm0.86$ & $11.80\pm0.86$ & $69.23\pm0.40$\tabularnewline
[0.05cm]
GEM & $72.74\pm0.91$ & $9.45\pm0.95$ & $80.53\pm0.28$ & $61.33\pm1.16$ & $9.21\pm0.94$ & $69.37\pm0.72$\tabularnewline
[0.05cm]
ER-Ring & $72.11\pm0.46$ & $9.53\pm0.23$ & $80.06\pm0.37$ & $61.96\pm1.22$ & $8.48\pm1.54$ & $69.14\pm0.87$\tabularnewline
\midrule
CTN~(+ITD) & $78.40\pm0.28$ & $6.25\pm0.41$ & $84.02\pm0.29$ & $67.7\pm60.96$ & $5.76\pm0.69$ & $72.58\pm0.62$\tabularnewline
[0.05cm]
CTN~(+BVFSM) & $77.78\pm0.32$ & $7.86\pm0.32$ & $85.03\pm0.28$ & $67.04\pm0.76$ & $7.61\pm0.47$ & $\pmb{74.01}\pm0.57$ \tabularnewline
[0.05cm]
CTN~(+Penalty) & $67.74\pm0.42$ & $10.10\pm0.48$ & $77.28\pm0.61$ & $47.41\pm2.93$ & $9.34\pm2.74$ & $55.76\pm1.64$\tabularnewline
[0.05cm]
CTN~(+BOME) & $\pmb{80.70}\pm0.26$ & $\pmb{4.73}\pm0.23$ & $\pmb{84.79}\pm0.25$ & $\pmb{68.16}\pm0.60$ & $\pmb{5.32}\pm0.76$ & $72.88\pm0.48$\tabularnewline
\bottomrule 
\end{tabular}
}
\caption{Results of continual learning as bilevel optimization. We compute the mean and standard error of each method's results over 5 independent runs.
Best results are \textbf{bolded}.
}
\label{tab:cl-full}
\end{table*}
\newpage

\section{Proof of the Result in Section \ref{sec:kkt}}
We proof Proposition \ref{pro:kkt} using Proposition 6.3 (presented below using our notation) in \citet{gong2021automatic} by checking all the assumptions required by Proposition 6.3 in \citet{gong2021automatic} are satisfied. Specifically, it remains to show that for any $k$, $\lambda_k<\infty$, $\lim_{k\to\infty}\nabla q(v_{k},\th_{k})=0$ and $q$ is lower bounded (this is trivial as $q\ge 0$ by its definition), which we prove below.

Firstly, simple calculation shows that for any $k$,
\[
\lambda_{k}=\left[\frac{-\left\langle \nabla f(v_{k},\th_{k}),\nabla q(v_{k},\th_{k})\right\rangle }{||\nabla q(v_{k},\th_{k})||^{2}}\right]_+\le\frac{\sup_{v,\th}||\nabla f(v,\th)||}{||\nabla q(v_{k},\th_{k})||}<\infty.
\]
Here the last inequality is by $||\nabla q(v_k,\th_k)||>0$.
Secondly, note that as we assume $\nabla q$ is continuous, this implies that 
\[
\lim_{k\to\infty}\nabla q(v_{k},\th_{k})=\nabla q(v^{*},\th^{*}).
\]
As $q(v^{*},\th^{*})=0$, we have $\nabla q(v^{*},\th^{*})=0$.
Using Proposition 6.3 in \citet{gong2021automatic} gives the desired result.

\begin{pro}[Proposition 6.3 in \citet{gong2021automatic}]
Assume $f,q,\nabla q$ are continuously differentiable. Let $\{[v_{k},\th_{k},\lambda_{k}]:k=1,2,...\}$
be a sequence which satisfies $\lim_{k\to\infty}||\nabla q(v_{k},\th_{k})||=0$
and $\lim_{k\to\infty}||\nabla f(v_{k},\th_{k})+\lambda_{k}\nabla q(v_{k},\th_{k})||=0$.
Assume that $[v^{*},\th^{*}]$ is a limit point of $[v_{k},\th_{k}]$
as $k\to\infty$ and $[v^{*},\th^{*}]$ satisfies CRCQ with $\nabla_{\th}q$,
then there exists a vector-valued Lagrange multiplier $\omega^{*} \in \RR^m$
(the same length as $\th$) such that 
\[
\nabla f(v^{*},\th^{*})+\nabla(\nabla_{\th}q(v^{*},\th^{*}))\omega^{*}=0.
\]

\end{pro}

\section{Proof of the Result in Section \ref{sec:pl_theory}} \label{apx:pl_theory}
We define $L_{q}:=2L(L/\kappa+1)$ and using Assumption
\ref{asm:Smoothness} and \ref{asm:PL-inequality}, we are able to show that
$q(v,\th)$ is $L_{q}$-smooth (see Lemma \ref{lem:q smoothness} for details).
For simplicity, we also assume that $\xi\le$1 throughout the proof. We use $b$ with some subscript to denote some general $O(1)$ constant and refer reader to section \ref{misc:constant} for their detailed value.

Note that $\hat{q}$ defined in Section \ref{sec:method} changes in different iterations (as it depends on $\th_k^{(T)}$) and so does $\nabla \hat{q}$. To avoid the confusion, we introduce several new notations. Firstly, given $v$ and $\th$, $\th^{(T)}$ denotes the results of $T$ steps of gradient of $g(v,\cdot)$ w.r.t. $\th$ starting from $\th$ with step size $\alpha$ (similar to the definition in (\ref{equ:thetaTT})). Note that $\th^{(T)}$ depends on $v$, $\th$ and $\alpha$. Our notation does not reflects this dependency on $v,\alpha$ as we find it introduces no ambiguity while much simplifies the notation. Also note that when taking gradient on $\hat{q}$, the $\th_k^{(T)}$ at iteration $k$ is treated as a constant and the gradient does not pass through it. To be clear, we define $\hdq(v,\th)=\nabla g(v,\th)-\left[\nabla_{1}^{\top}g(v,\th^{(T)}),\textbf{0}^{\top}\right]^{\top}$, where $\textbf{0}$ denotes a zero vector with the same dimension as $\th$. Using this definition, $\hdq(v_k,\th_k) = \nabla \hat{q}(v_k,\th_k)$ at iteration $k$. We also let $\lm(v,\th)$ be the solution of the dual problem of 
\begin{equation} \label{eq:proof_primal}
    \min_{\delta}||\hdq(v,\th)-\nabla f(v,\th)||^{2}\ s.t.\ \left\langle \hdq(v,\th),\nabla f(v,\th)\right\rangle \ge\ul||\hdq(v,\th)||^{2}.
\end{equation}
That is 
\begin{equation} \label{eq:local solution}
\lm(v,\th)=\begin{cases}
\frac{[\ul||\hdq(v,\th)||^{2}-\left\langle \hdq(v,\th),\nabla f(v,\th)\right\rangle ]_{+}}{||\hdq(v,\th)||^{2}} & \text{when}\ ||\hdq(v,\th)||>0\\
0 & \text{when}\ ||\hdq(v,\th)||=0
\end{cases}
\end{equation}
We might use $\lm$ for $\lm(v,\th)$ when it introduces no confusion. Also, denote $\d(v,\th) = \lm(v,\th) \hdq(v,\th) + \nabla f(v,\th)$ and thus $\delta_k = \d(v_k,\th_k)$.

We start with several technical Lemmas showing some basic function
properties.

\subsection{Technical Lemmas}

\begin{lemma} \label{lem:Lower bound}
Under Assumption \ref{asm:PL-inequality}, for any $v,\th$, $g(v,\th)-g(v,\th^*(v))\ge\frac{\kappa}{4}||\th-\th^{*}(v)||^{2}$.
\end{lemma}

\begin{lemma} \label{lem:bound hat dq}
Under Assumption \ref{asm:PL-inequality} and \ref{asm:Smoothness}, we have $||\nabla q(v,\th)-\hdq(v,\th)||\le L||\th^{(T)}-\th^{*}(v)||$ for any $v,\th$.
Also, when $||\hat{\nabla}q(v,\th)||=0$,
$q(v,\th)=0$.
\end{lemma}

\begin{lemma} \label{lem:th implicit lipschitz}
Under Assumption \ref{asm:Smoothness}, \ref{asm:PL-inequality},
we have $||\th^{*}(v_{2})-\th^{*}(v_{1})||\le\frac{2L}{\kappa}||v_{1}-v_{2}||.$
\end{lemma}

\begin{lemma} \label{lem:q smoothness}
Under Assumption \ref{asm:Smoothness}, we have $||\nabla_{\th}q(v,\th_{1})-\nabla_{\th}q(v,\th_{2})||\le L||\th_{1}-\th_{2}||$,
for any $v$. Further assume Assumption \ref{asm:PL-inequality}, we have 
\[
\left\Vert \nabla q(v_{1},\th_{1})-\nabla q(v_{2},\th_{2})\right\Vert \le L_{q}||[v_{1},\th_{1}]-[v_{2},\th_{2}]||,
\]
where $L_{q}:=2L(L/\kappa+1)$.
\end{lemma}

\begin{lemma} \label{lem:opt th_T}
Under Assumption \ref{asm:PL-inequality}, \ref{asm:Smoothness} and assume
that $\alpha<2/L$. Given any $v,\th$, suppose $\th^{(0)}=\th$ and $\th^{(t+1)}=\th^{(t)}-\alpha\nabla_{\th}q(v,\th^{(t)})$,
then for any $t$, we have $q(v,\th^{(t)})\le\exp(-\bi(\alpha,L,\kappa)t)q(v,\th)$,
where $\bi(\alpha,L,\kappa)=$ is some strictly positive constant that
depends on $\alpha$, $L$ and $\kappa$.
\end{lemma}

\begin{lemma} \label{lem:bound d}
Under Assumption \ref{asm:Boundedness}, for any $[v,\th]$, we
have $||\d(v,\th)||,||\nabla q(v,\th)||,||\hdq(v,\th)||\le\bii(M,\ul )$,
where $\bii(M,\ul )=(3+\ul)M$.
\end{lemma}

\begin{lemma} \label{lem:bound lambda psi}
Under Assumption \ref{asm:Boundedness}, for any $[v,\th]$, we
have $\lm||\hdq||^{2}\le \ul ||\hdq||^{2}+M||\hdq||,$where
$\lm$ are defined in (\ref{eq:local solution}).
\end{lemma}

\begin{lemma} \label{lem:bound dq}
Under Assumption \ref{asm:PL-inequality} and \ref{asm:Smoothness},
we have $||\nabla q(v,\th)||\le2\kappa^{-1/2}L_{q}q^{1/2}(v,\th).$
\end{lemma}

\subsubsection{Lemmas}
Now we give several main lemmas that are used to prove the result
in Section \ref{sec:pl_theory}.

\begin{lemma} \label{lem:one step descent of q}
Under Assumption \ref{asm:PL-inequality}, \ref{asm:Smoothness} and \ref{asm:Boundedness}, when
$||\hdq(v_{k},\th_{k})||>0$, we have
\begin{align*}
q(v_{k+1},\th_{k+1})-q(v_{k},\th_{k}) & \le-\xi \ul ||\nabla q(v_{k},\th_{k})||^{2}+\xi \ul L_{q}||\th_{k}^{(T)}-\th^{*}(v_{k})||\ (L_{q}||\th_{k}^{(T)}-\th^{*}(v_{k})||+2L_{q}||\th_{k}-\th^{*}(v_{k})||)\\
 & +\xi\bii L||\th_{k}^{(T)}-\th^{*}(v_{k})||+L_{q}\xi^{2}\bii^{2}/2.
\end{align*}
When $||\hdq(v_{k},\th_{k})||=0$, we have $q(v_{k+1},\th_{k+1})-q(v_{k},\th_{k})\le\xi^{2}L_{q}\bii^{2}/2$.
\end{lemma}

\begin{lemma} \label{lem:decay q}
Under Assumption \ref{asm:PL-inequality}, \ref{asm:Smoothness} and \ref{asm:Boundedness}, choosing $T\ge\biii(\ul ,r,\kappa,L)$,
we have 
\[
q(v_{k},\th_{k})\le\exp(-\biv k)q(v_{0},\th_{0})+\Delta,
\]
where $\biv=-\log(1-\frac{\xi}{4}\ul \kappa)$ is some strictly positive
constant and $\Delta=O(\exp(-\bi T)+\xi)$.
\end{lemma}

\begin{lemma} \label{lem:decay hdq}
Under Assumption \ref{asm:PL-inequality}, \ref{asm:Smoothness} and \ref{asm:Boundedness}, we have 
\[
\sum_{k=0}^{K-1}||\nabla q(v_{k},\th_{k})||^{2}\le\frac{\bv q(v_{0},\th_{0})}{\xi}+K\xi^{2}\bvi\Delta,
\]
where $\bv$ is some constant depends on $L_{q},\ul ,\kappa$; $\bvi$ is some constant depends on $\kappa, L$ and $\Delta$ is defined in Lemma \ref{lem:decay q}.
\end{lemma}

\begin{lemma} \label{lem:Convergence hk}
Under Assumption \ref{asm:PL-inequality}, \ref{asm:Smoothness} and \ref{asm:Boundedness}, choosing $T\ge\biii(\ul ,r,\kappa,L)$
and assume that $r,\xi\le1/L$, we have 

\[
\sum_{k=0}^{K-1}\left[||\d(v_{k},\th_{k})||^{2}+q(v_{k},\th_{k})\right]=O(\xi^{-1}+K\exp(-\bi T/2)+K\xi^{1/2}+\xi^{-1/2}K^{1/2}q^{1/2}(v_{0},\th_{0})).
\]
\end{lemma}

\subsection{Proof of Theorem \ref{thm:Convergence k}}

Using our definition of $\lm$ in (\ref{eq:local solution}), we have 
\begin{align*}
||\nabla f(v,\th)+\lm(v,\th)\nabla q(v,\th)|| & \le||\nabla f(v,\th)+\lm(v,\th)\hdq(v,\th)||+||\lm(v,\th)(\hdq(v,\th)-\nabla q(v,\th))||\\
 & =||\d(v,\th)||+||\lm(v,\th)(\hdq(v,\th)-\nabla q(v,\th))||.
\end{align*}

Using Lemma \ref{lem:bound hat dq}, we know that when $||\hdq||=0$, we have $q=0$ and thus $||\nabla q||=0$. In this case, $||\lm(\hdq-\nabla q)||=0$. When $||\hdq||>0$, some algebra shows that 
\begin{align*}
||\lm(\hdq-\nabla q)|| & \le\left[\ul -\left\langle \nabla f,\hdq/||\hdq||\right\rangle ||\hdq||^{-1}\right]||\hdq-\nabla q||\\
 & \le(\ul -\left\langle \nabla f,\hdq/||\hdq||\right\rangle ||\hdq||^{-1})||\hdq-\nabla q||.
\end{align*}
Notice that 
\begin{align*}
||\hdq(v,\th)-\nabla q(v,\th)|| & \le L||\th^{(T)}-\th^{*}(v)||\\
 & \le2L\kappa^{-1/2}q^{1/2}(v,\th^{(T)})\\
 & \le2L\kappa^{-1/2}\exp(-\bi T/2)q^{1/2}(v,\th)\\
 & \le2L\kappa^{-1}\exp(-\bi T/2)||\nabla q(v,\th)||.
\end{align*}
Here the first inequality is by Lemma \ref{lem:bound hat dq}, the second inequality is by Lemma
\ref{lem:Lower bound}, the third inequality is by Lemma \ref{lem:opt th_T} and the last inequality is by Assumption \ref{asm:PL-inequality} (using $||\nabla q(v,\th)||\ge||\nabla_{\th} g(v,\th)||$).
Similarly, under assumption that $T\ge\left\lceil -b_{1}^{-1}\log(\frac{1}{16}\kappa^{2}L^{-2})\right\rceil $,
$L\kappa^{-1}\exp(-\bi T/2)\le1/4$, 
\begin{align*}
||\hdq(v,\th)|| & =||\hdq(v,\th)-\nabla q(v,\th)+\nabla q(v,\th)||\\
 & \ge||\nabla q(v,\th)||-||\hdq(v,\th)-\nabla q(v,\th)||\\
 & \ge||\nabla q(v,\th)||(1-(2L\kappa^{-1}\exp(-\bi T/2)))\\
 & \ge\frac{1}{2}||\nabla q(v,\th)||.
\end{align*}
This implies that 
\[
\frac{||\hdq-\nabla q||}{||\hdq||}\le2\frac{||\hdq-\nabla q||}{||\nabla q||}\le4L\kappa^{-1}\exp(-\bi T/2).
\]
We thus have
\begin{align*}
||\lm(v,\th)(\hdq(v,\th)-\nabla q(v,\th))|| & \le \ul ||\hdq-\nabla q||+\left\langle \nabla f,\frac{\hdq}{||\hdq||}\right\rangle \frac{||\hdq-\nabla q||}{||\hdq||}\\
 & \le2L\kappa^{-1}\exp(-\bi T/2)\left[\ul ||\nabla q(v,\th)||+2\left\langle \nabla f,\frac{\hdq}{||\hdq||}\right\rangle \right]\\
 & \le2L\kappa^{-1}\exp(-\bi T/2)(\ul +2)\bii,
\end{align*}
where the last inequality is by Lemma \ref{lem:bound d}.
Combining all the results and using $||\nabla q(v_{k},\th_{k})||\le2\kappa^{-1/2}L_{q}q^{1/2}(v_{k},\th_{k})$
by Lemma \ref{lem:bound dq}, we have 
\begin{align*}
\K(v,\th) & \le||\nabla f(v,\th)+\lm(v,\th)\nabla q(v,\th)||^{2}+q(v,\th)\\
 & \le2||\nabla f(v,\th)+\lm(v,\th)\hdq(v,\th)||^{2}+q(v,\th)+2||\lm(v,\th)(\hdq(v,\th)-\nabla q(v,\th))||^{2}\\
 & \le2||\d(v,\th)||^{2}+q(v,\th)+8L^{2}\kappa^{-2}\exp(-\bi T)(\ul +2)^{2}\bii^{2}.
\end{align*}
Using Lemma \ref{lem:Convergence hk}, we have 
\begin{align*}
\min_{k}\K(v_{k},\th_{k}) & =O(\min_{k}(||\d(v_{k},\th_{k})||^{2}+q(v_{k},\th_{k}))+\exp(-\bi T))\\
 & =O(\xi^{-1}+K\exp(-\bi T/2)+K\xi^{1/2}+\xi^{-1/2}K^{1/2}q^{1/2}(v_{0},\th_{0})).
\end{align*}

\subsection{Proof of Lemmas}

\subsubsection{Proof of Lemma \ref{lem:one step descent of q}}

When $||\hdq(v_{k},\th_{k})||>0$, by Lemma \ref{lem:q smoothness}, we
know that $q$ is $L_{q}$-smoothness, we have 
\begin{align*}
q(v_{k+1},\th_{k+1})-q(v_{k},\th_{k}) & \le-\xi\left\langle \nabla q(v_{k},\th_{k}),\d(v_{k},\th_{k})\right\rangle +\frac{L_{q}\xi^{2}}{2}||\d(v_{k},\th_{k})||^{2}\\
 & \le-\xi\left\langle \hdq(v_{k},\th_{k}),\d(v_{k},\th_{k})\right\rangle -\xi\left\langle \nabla q(v_{k},\th_{k})-\hdq(v_{k},\th_{k}),\d(v_{k},\th_{k})\right\rangle +L_{q}\xi^{2}\bii^{2}/2\\
 & \le-\xi \ul ||\hdq(v_{k},\th_{k})||^{2}-\xi\left\langle \nabla q(v_{k},\th_{k})-\hdq(v_{k},\th_{k}),\d(v_{k},\th_{k})\right\rangle +L_{q}\xi^{2}\bii^{2}/2\\
 & \le-\xi \ul ||\hdq(v_{k},\th_{k})||^{2}+\xi\bii||\nabla q(v_{k},\th_{k})-\hdq(v_{k},\th_{k})||+L_{q}\xi^{2}\bii^{2}/2.
\end{align*}
where the second and the last inequality is by Lemma \ref{lem:bound d}
and the third inequality is ensured by the constraint in the local
subproblem ($\left\langle \nabla \hdq(v_{k},\th_{k}),\d(v_{k},\th_{k})\right\rangle \ge \eta || \hdq(v_k,\th_k)^2||$.). And by Lemma \ref{lem:bound hat dq}, we have $||\nabla q(v_{k},\th_{k})-\hdq(v_{k},\th_{k})||\le L||\th_{k}^{(T)}-\th^{*}(v_{k})||$.
Plug in the bound we have 
\[
q(v_{k+1},\th_{k+1})-q(v_{k},\th_{k})\le-\xi \ul ||\hdq(v_{k},\th_{k})||^{2}+\xi\bii L||\th_{k}^{(T)}-\th^{*}(v_{k})||\ .
\]
Also notice that 
\begin{align*}
\left|||\hdq(v_{k},\th_{k})||^{2}-||\nabla q(v_{k},\th_{k})||^{2}\right| & \le||\hdq(v_{k},\th_{k})-\nabla q(v_{k},\th_{k})||\ ||\hdq(v_{k},\th_{k})+\nabla q(v_{k},\th_{k})||\\
 & \le||\hdq(v_{k},\th_{k})-\nabla q(v_{k},\th_{k})||\ (||\hdq(v_{k},\th_{k})-\nabla q(v_{k},\th_{k})||+2||\nabla q(v_{k},\th_{k})||)\\
 & \le L_{q}||\th_{k}^{(T)}-\th^{*}(v_{k})||\ (L_{q}||\th_{k}^{(T)}-\th^{*}(v_{k})||+2||\nabla q(v_{k},\th_{k})||)\\
 & =L_{q}||\th_{k}^{(T)}-\th^{*}(v_{k})||\ (L_{q}||\th_{k}^{(T)}-\th^{*}(v_{k})||+2||\nabla q(v_{k},\th_{k})-\nabla q(v_{k},\th^{*}(v_{k}))||)\\
 & \le L_{q}||\th_{k}^{(T)}-\th^{*}(v_{k})||\ (L_{q}||\th_{k}^{(T)}-\th^{*}(v_{k})||+2L_{q}||\th_{k}-\th^{*}(v_{k})||),
\end{align*}
where the third inequality is by Lemma \ref{lem:bound hat dq}, the equality is by $\nabla q(v_k, \th^*(v_k))=0$ and the
last inequality is by Lemma \ref{lem:q smoothness}.

Using this bound, we further have
\begin{align*}
q(v_{k+1},\th_{k+1})-q(v_{k},\th_{k}) & \le-\xi \ul ||\nabla q(v_{k},\th_{k})||^{2}+\xi \ul \left|||\hdq(v_{k},\th_{k})||^{2}-||\nabla q(v_{k},\th_{k})||^{2}\right|\\
 & +\xi\bii||\nabla q(v_{k},\th_{k})-\hdq(v_{k},\th_{k})||+L_{q}\xi^{2}\bii^{2}/2\\
 & \le-\xi \ul ||\nabla q(v_{k},\th_{k})||^{2}+\xi \ul L_{q}||\th_{k}^{(T)}-\th^{*}(v_{k})||\ (L_{q}||\th_{k}^{(T)}-\th^{*}(v_{k})||+2L_{q}||\th_{k}-\th^{*}(v_{k})||)\\
 & +\xi\bii L||\th_{k}^{(T)}-\th^{*}(v_{k})||+L_{q}\xi^{2}\bii^{2}/2.
\end{align*}
When $||\hdq(v_{k},\th_{k})||=0$, by Lemma \ref{lem:bound hat dq}, $q(v_{k},\th_{k})=0$
and hence $\nabla q(v_{k},\th_{k})=0$. We thus have 
\begin{align*}
q(v_{k+1},\th_{k+1})-q(v_{k},\th_{k}) & \le-\xi\left\langle \nabla q(v_{k},\th_{k}),\d(v_{k},\th_{k})\right\rangle +\frac{L_{q}\xi^{2}}{2}||\d(v_{k},\th_{k})||^{2}\\
 & =\frac{L_{q}\xi^{2}}{2}||\d(v_{k},\th_{k})||^{2}\\
 & \le\xi^{2}L_{q}\bii^{2}/2.
\end{align*}

\subsubsection{Proof of Lemma \ref{lem:decay q}}

By Lemma \ref{lem:one step descent of q}, when $||\hdq(v_{k},\th_{k})||>0$,
we have 
\begin{align*}
q(v_{k+1},\th_{k+1})-q(v_{k},\th_{k}) & \le-\xi \ul ||\nabla q(v_{k},\th_{k})||^{2}+\xi \ul L_{q}||\th_{k}^{(T)}-\th^{*}(v_{k})||\ (L_{q}||\th_{k}^{(T)}-\th^{*}(v_{k})||+2L_{q}||\th_{k}-\th^{*}(v_{k})||)\\
 & +\xi\bii L||\th_{k}^{(T)}-\th^{*}(v_{k})||+L_{q}\xi^{2}\bii^{2}/2.
\end{align*}
By Lemma \ref{lem:Lower bound} and Lemma \ref{lem:opt th_T}
\begin{align*}
||\th_{k}^{(T)}-\th^{*}(v_{k})|| & \le2\kappa^{-1/2}q^{1/2}(v_{k},\th_{k}^{(T)})\le2\kappa^{-1/2}\exp(-\bi T/2)q^{1/2}(v_{k},\th_{k}).\\
||\th_{k}-\th^{*}(v_{k})|| & \le2\kappa^{-1/2}q^{1/2}(v_{k},\th_{k}).
\end{align*}
Using those bounds, we know that 
\[
L_{q}||\th_{k}^{(T)}-\th^{*}(v_{k})||\ (L_{q}||\th_{k}^{(T)}-\th^{*}(v_{k})||+2L_{q}||\th_{k}-\th^{*}(v_{k})||)\le12L_{q}^{2}\kappa^{-1}\exp(-\bi T)q(v_{k},\th_{k})
\]
This implies that 
\begin{align*}
 & q(v_{k+1},\th_{k+1})-q(v_{k},\th_{k})\\
\le & -\xi \ul ||\nabla q(v_{k},\th_{k})||^{2}+12\xi \ul L_{q}^{2}\kappa^{-1}\exp(-\bi T)q(v_{k},\th_{k})\\
+ & 2\xi\bii L\kappa^{-1/2}\exp(-\bi T/2)q^{1/2}(v_{k},\th_{k})+L_{q}\xi^{2}\bii^{2}/2\\
\le & -\xi \ul \kappa q(v_{k},\th_{k})+12\xi \ul L_{q}^{2}\kappa^{-1}\exp(-\bi T)q(v_{k},\th_{k})\\
+ & 2\xi\bii L\kappa^{-1/2}\exp(-\bi T/2)q^{1/2}(v_{k},\th_{k})+L_{q}\xi^{2}\bii^{2}/2.
\end{align*}
Choosing $T$ such that $T\ge\biii(\ul ,\alpha,\kappa,L)$ where 
\[
\biii(\ul ,\alpha,\kappa,L)=\left\lceil -b_{1}^{-1}\log(\frac{\ul \kappa}{64\ul L_{q}^{2}})\right\rceil,
\]
we have 
\[
q(v_{k+1},\th_{k+1})-q(v_{k},\th_{k})\le-\frac{3}{4}\xi \ul \kappa q(v_{k},\th_{k})+2\xi\bii L\kappa^{-1/2}\exp(-\bi T/2)q^{1/2}(v_{k},\th_{k})+L_{q}\xi^{2}\bii^{2}/2.
\]
This implies that when $\frac{64\bii^{2}L^{2}}{\ul ^{2}\kappa}\exp(-\bi T)\le q(v_{k},\th_{k})$
and $\frac{2L_{q}\xi\bii^{2}}{\ul \kappa}\le q(v_{k},\th_{k})$,
\[
q(v_{k+1},\th_{k+1})-q(v_{k},\th_{k})\le-\frac{1}{4}\xi \ul \kappa q(v_{k},\th_{k}).
\]
Let $a=\max(\frac{64\bii^{2}L^{2}}{\ul ^{2}\kappa}\exp(-\bi T),\frac{2L_{q}\xi\bii^{2}}{\ul \kappa})$.
Also, when $q(v_{k},\th_{k})<a$, 
\begin{align*}
q(v_{k+1},\th_{k+1}) & \le q(v_{k},\th_{k})+2\xi\bii L\kappa^{-1/2}\exp(-\bi T/2)q^{1/2}(v_{k},\th_{k})+L_{q}\xi^{2}\bii^{2}/2\\
 & <a+2\xi\bii L\kappa^{-1/2}\exp(-\bi T/2)\sqrt{a}+L_{q}\xi^{2}\bii^{2}/2.
\end{align*}
Note that 
\begin{align*}
2\xi\bii L\kappa^{-1/2}\exp(-\bi T/2) & \le\frac{\xi \ul \kappa}{4}\sqrt{a}\\
L_{q}\xi^{2}\bii^{2}/2 & \le\frac{\xi \ul \kappa}{4}a.
\end{align*}
This gives that in the case of $q(v_k,\th_k)<a$,
\[
q(v_{k+1},\th_{k+1})<(1+\frac{\xi \ul \kappa}{4})a.
\]
Define $k_{0}$ as the first iteration such that $q(v_{k},\th_{k})<a$.
This implies that, for any $k\le k_{0}$, 
\[
q(v_{k},\th_{k})\le(1-\frac{\xi}{4}\ul \kappa)^{k}q(v_{0},\th_{0}).
\]
When any $k>k_{0}$, we show that $q(v_{k+1},\th_{k+1})\le(1+\frac{\xi \ul \kappa}{4})a$.
This can be proved by induction. At $k=k_{0}+1$, if $q(v_{k},\th_{k})<a,$
we have $q(v_{k},\th_{k})<(1+\frac{\xi \ul \kappa}{4})a$. Else if
at $k=k_{0}+1$, $q(v_{k},\th_{k})\ge a$, $q(v_{k+1},\th_{k+1})\le q(v_{k},\th_{k})\le a$.
We thus have the conclusion that for any $k>k_{0}$, $q(v_{k},\th_{k})\le(1+\frac{\ul \kappa}{4})a$.
Combining the result, we have 
\[
q(v_{k},\th_{k})\le(1-\frac{\xi}{4}\ul \kappa)^{k}q(v_{0},\th_{0})+\Delta,
\]
where we denote 
\begin{equation} \label{eq:proof_delta}
    \Delta=(1+\frac{\ul \kappa}{4})(\frac{64\bii^{2}L^{2}}{\ul ^{2}\kappa^{3}}\exp(-\bi T)+\frac{2L_{q}\xi\bii^{2}}{\ul \kappa})+L_{q}\xi^{2}\bii^{2}/2=O(\exp(-\bi T)+\xi).
\end{equation}

Let $\biv(\ul ,\kappa,\xi)=-\log(1-\frac{\xi}{4}\ul \kappa)$, we
have the desired result.

\subsubsection{Proof of Lemma \ref{lem:decay hdq}}

By Lemma \ref{lem:bound dq} and \ref{lem:decay q}, we have
\begin{align*}
||\nabla q(v_{k},\th_{k})||^{2} & \le2\kappa^{-1}L_{q}^{2}q(v_{k},\th_{k})\\
 & \le2\kappa^{-1}L_{q}^{2}\left[\exp(-\biv k)q(v_{0},\th_{0})+\Delta\right],
\end{align*}
where $\Delta$ is defined in (\ref{eq:proof_delta}).
Also notice that 
\begin{align*}
||\hdq(v,\th)|| & \le||\hdq(v,\th)-\nabla q(v,\th)||+||\nabla q(v,\th)||\\
 & \le L||\th^{(T)}-\th^{*}(v)||+||\nabla q(v,\th)||\\
 & \le2L\kappa^{-1/2}q^{1/2}(v,\th^{(T)})+||\nabla q(v,\th)||\\
 & \le2L\kappa^{-1/2}\exp(-\bi T/2)q^{1/2}(v,\th)+||\nabla q(v,\th)||\\
 & \le(2L\kappa^{-1}\exp(-\bi T/2)+1)||\nabla q(v,\th)||\\
 & \le(2L\kappa^{-1}+1)||\nabla q(v,\th)||
\end{align*}
Here the first inequality is by triangle inequality, the second inequality
is by Lemma \ref{lem:bound hat dq}, the third inequality is by Lemma
\ref{lem:Lower bound}, the forth inequality is by Lemma \ref{lem:opt th_T}
and the fifth inequality is by Assumption \ref{asm:PL-inequality}. Taking
summation over iteration and using Lemma \ref{lem:decay q}, we have 
\begin{align*}
\sum_{k=0}^{K-1}||\hdq(v,\th)||^{2} & \le(2L\kappa^{-1}+1)^{2}\sum_{k=0}^{K-1}||\nabla q(v_{k},\th_{k})||^{2}\\
 & \le(2L\kappa^{-1}+1)^{2}\left[2\kappa^{-1}L_{q}^{2}q(v_{0},\th_{0})\sum_{k=0}^{K-1}\left[\exp(-\biv k)\right]+K\Delta\right]\\
 & \le(2L\kappa^{-1}+1)^{2}\left[\frac{2\kappa^{-1}L_{q}^{2}q(v_{0},\th_{0})}{1-\exp(-\biv)}+K\Delta\right]\\
 & =\frac{\bv q(v_{0},\th_{0})}{\xi}+K\bvi\Delta,
\end{align*}
where we define $\bv(L_{q},\ul ,\kappa)=\frac{16L_{q}^{2}}{\ul \kappa^{2}}(2L\kappa^{-1}+1)^{2}$
and $\bvi(\kappa,L)=(2L\kappa^{-1}+1)^{2}$.

\subsection{Proof of Lemma \ref{lem:Convergence hk}}
Remind that by our definition of $\lm$ in (\ref{eq:local solution}) and Assumption \ref{asm:Smoothness}, we have
\begin{align*}
f(v_{k+1},\th_{k+1})-f(v_{k},\th_{k}) & \le-\xi\left\langle \nabla f(v_{k},\th_{k}),\d(v_{k},\th_{k})\right\rangle +\frac{L\xi^{2}}{2}||\d(v_{k},\th_{k})||^{2}\\
 & =-\xi\left\langle \d(v_{k},\th_{k})-\lm(v_{k},\th_{k})\hdq(v_{k},\th_{k}),\d(v_{k},\th_{k})\right\rangle +\frac{L\xi^{2}}{2}||\d(v_{k},\th_{k})||^{2}\\
 & =-(\xi-\frac{L\xi^{2}}{2})||\d(v_{k},\th_{k})||^{2}+\xi\lm(v_{k},\th_{k})\left\langle \hdq(v_{k},\th_{k}),\d(v_{k},\th_{k})\right\rangle \\
 & \le-(\xi-\frac{L\xi^{2}}{2})||\d(v_{k},\th_{k})||^{2}+\xi \ul \lm(v_{k},\th_{k})||\hdq(v_{k},\th_{k})||^{2}\\
 & \le-\frac{\xi}{2}||\d(v_{k},\th_{k})||^{2}+\xi \ul \lm(v_{k},\th_{k})||\hdq(v_{k},\th_{k})||^{2},
\end{align*}
where the last inequality is by the assumption on $\xi\le1/L$. To show the second inequality, we use the complementary slackness of Problem (\ref{eq:proof_primal}), that is
\[
\lm(v_{k},\th_{k})\left[\left\langle \hdq(v_{k},\th_{k}),\d(v_{k},\th_{k})\right\rangle -\ul||\hdq(v_{k},\th_{k})||\right]=0.
\]

By telescoping, 
\begin{align*}
\sum_{k=0}^{K-1}f(v_{k+1},\th_{k+1})-f(v_{k},\th_{k}) & \le-\frac{\xi}{2}\sum_{k=0}^{K-1}||\d(v_{k},\th_{k})||^{2}+\xi \ul \sum_{k=0}^{K-1}\lm(v_{k},\th_{k})||\hdq(v_{k},\th_{k})||^{2}\\
 & \le-\frac{\xi}{2}\sum_{k=0}^{K-1}||\d(v_{k},\th_{k})||^{2}+\xi \ul \sum_{k=0}^{K-1}(\ul ||\hdq(v_{k},\th_{k})||^{2}+M||\hdq(v_{k},\th_{k})||)\\
 & =-\frac{\xi}{2}\sum_{k=0}^{K-1}||\d(v_{k},\th_{k})||^{2}+\xi \ul ^{2}\sum_{k=0}^{K-1}||\hdq(v_{k},\th_{k})||^{2}+\xi \ul M\sum_{k=0}^{K-1}||\hdq(v_{k},\th_{k})||\\
 & \le-\frac{\xi}{2}\sum_{k=0}^{K-1}||\d(v_{k},\th_{k})||^{2}+\xi \ul ^{2}\sum_{k=0}^{K-1}||\hdq(v_{k},\th_{k})||^{2}+\xi \ul M\sqrt{K}\sqrt{\sum_{k=0}^{K-1}||\hdq(v_{k},\th_{k})||^{2}},
\end{align*}
where the second inequality is by Lemma \ref{lem:bound lambda psi} and the last inequality is by Holder's inequality. Since
$\sum_{k=0}^{K-1}f(v_{k+1},\th_{k+1})-f(v_{k},\th_{k})=f(v_{K},\th_{K})-f(v_{0},\th_{0})$,
rearrange the terms, we have 
\[
\xi\sum_{k=0}^{K-1}||\d(v_{k},\th_{k})||^{2}\le2(f(v_{0},\th_{0})-f(v_{K},\th_{K}))+2\xi \ul ^{2}\sum_{k=0}^{K-1}||\hdq(v_{k},\th_{k})||^{2}+2\xi \ul M\sqrt{K}\sqrt{\sum_{k=0}^{K-1}||\hdq(v_{k},\th_{k})||^{2}}.
\]
This implies that 
\begin{align*}
\xi\sum_{k=0}^{K-1}\left[||\d(v_{k},\th_{k})||^{2}+q(v_{k},\th_{k})\right] & \le2(f(v_{0},\th_{0})-f(v_{K},\th_{K}))+2\xi \ul ^{2}\sum_{k=0}^{K-1}||\hdq(v_{k},\th_{k})||^{2}\\
 & +2\xi \ul M\sqrt{K}\sqrt{\sum_{k=0}^{K-1}||\hdq(v_{k},\th_{k})||^{2}}+\xi\sum_{k=0}^{K-1}q(v_{k},\th_{k}).
\end{align*}
 Using Lemma \ref{lem:decay q}, we know that 
\[
 q(v_{k},\th_{k})\le(1-\frac{\xi}{4}\ul \kappa)^{k}q(v_{0},\th_{0})+\Delta.
\]
This gives that 
\[
\xi\sum_{k=0}^{K-1}q(v_{k},\th_{k})\le\frac{4q(v_{0},\th_{0})}{\ul \kappa}+\xi K\Delta.
\]
Using Lemma \ref{lem:decay hdq}, \ref{lem:decay q} and $\sqrt{x+y}\le\sqrt{x}+\sqrt{y}$,
we have 
\begin{align*}
2\xi \ul ^{2}\sum_{k=1}^{K}||\hdq(v_{k},\th_{k})||^{2} & \le2\ul ^{2}\bv q(v_{0},\th_{0})+2K\ul ^{2}\xi\bvi\Delta\\
2\xi \ul M\sqrt{K}\sqrt{\sum_{k=1}^{K}||\hdq(v_{k},\th_{k})||^{2}} & \le2\xi^{1/2}K^{1/2}\bv^{1/2}\ul Mq^{1/2}(v_{0},\th_{0})+2K\xi\bvi^{1/2}\ul M\Delta^{1/2}\\
\end{align*}
This implies that
\begin{align*}
 & \xi\sum_{k=0}^{K-1}\left[||\d(v_{k},\th_{k})||^{2}+q(v_{k},\th_{k})\right]\\
\le & 2(f(v_{0},\th_{0})-f(v_{K},\th_{K}))+2\ul ^{2}\bv q(v_{0},\th_{0})+2K\ul ^{2}\xi\bvi\Delta+2\xi^{1/2}K^{1/2}\bv^{1/2}\ul Mq^{1/2}(v_{0},\th_{0})\\
+ & 2K\xi\bvi^{1/2}\ul M\Delta^{1/2}+\frac{4q(v_{0},\th_{0})}{\ul \kappa}+\xi K\Delta\\
\le & 2(f(v_{0},\th_{0})-f(v_{K},\th_{K}))+(2\ul ^{2}\bv+\frac{4}{\ul \kappa})q(v_{0},\th_{0})+2K\xi(\bvi^{1/2}\ul M\Delta^{1/2}+(\bvi \ul ^{2}+1/2)\Delta)\\
+ & 2\xi^{1/2}K^{1/2}\bv^{1/2}\ul Mq^{1/2}(v_{0},\th_{0})
\end{align*}
We thus have
\begin{align*}
 & \sum_{k=0}^{K-1}\left[||\d(v_{k},\th_{k})||^{2}+q(v_{k},\th_{k})\right]\\
 = & O(\xi^{-1}+K\Delta^{1/2}+\xi^{-1/2}K^{1/2}q^{1/2}(v_{0},\th_{0}))\\
 = & O(\xi^{-1}+K\exp(-\bi T/2)+K\xi^{1/2}+\xi^{-1/2}K^{1/2}q^{1/2}(v_{0},\th_{0})).
\end{align*}

\subsection{Proofs of Technical Lemmas}

\subsubsection{Proof of Lemma \ref{lem:Lower bound}}
Please see the proof of Theorem 2 in \citet{karimi2016linear}.

\subsubsection{Proof of Lemma \ref{lem:bound hat dq}}

Since $\nabla_{2}g(v,\th^{*}(v))=0$, we have $\nabla_{v}g(v,\th^{*}(v))=\nabla_{1}g(v,\th^{*}(v))+\nabla_{v}\th^{*}(v)\nabla_{2}g(v,\th^{*}(v))=\nabla_{1}g(v,\th^{*}(v))$. Thus

\[
\nabla q(v,\th)=\left[\begin{array}{c}
\nabla_{v}g(v,\th)-\nabla_{v}g(v,\th^{*}(v))\\
\nabla_{\th}g(v,\th)
\end{array}\right]=\left[\begin{array}{c}
\nabla_{v}g(v,\th)-\nabla_{1}g(v,\th^{*}(v))\\
\nabla_{\th}g(v,\th)
\end{array}\right].
\]
Also note that 
\[
\hat{\nabla}q(v,\th)=\left[\begin{array}{c}
\nabla_{v}g(v,\th)-\nabla_{1}g(v,\th^{(T)})\\
\nabla_{\th}g(v,\th)
\end{array}\right].
\]
This gives that 
\begin{align*}
||\nabla q(v,\th)-\hdq(v,\th)|| & =||\nabla_{1}g(v,\th^{(T)})-\nabla_{1}g(v,\th^{*}(v))||\\
 & \le L||\th^{(T)}-\th^{*}(v)||.
\end{align*}
Also when $0=||\hat{\nabla}q(v,\th)||=\sqrt{||\nabla_{v}g(v,\th)-\nabla_{1}g(v,\th^{(T)})||^{2}+||\nabla_{\th}g(v,\th)||^{2}}$,
we have $||\nabla_{\th}g(v,\th)||=0$. Under Assumption \ref{asm:PL-inequality},
\[
0=||\nabla_{\th}g(v,\th)||\ge\kappa(g(v,\th)-g(v,\th^{*}(v)))=\kappa q(v,\th).
\]

\subsubsection{Proof of Lemma \ref{lem:th implicit lipschitz}}

Using Assumption \ref{asm:PL-inequality} and $\nabla_{2}g(v_{1},\th^{*}(v_{1}))=0$,
we have 
\[
||\nabla_{2}g(v_{1},\th^{*}(v_{2}))||\ge\sqrt{\kappa(g(v_{1},\th^{*}(v_{2}))-g(v_{1},\th^{*}(v_{1}))}.
\]
Also by Lemma \ref{lem:Lower bound}, we have $g(v_{1},\th^{*}(v_{2}))-g(v_{1},\th^{*}(v_{1})\ge\frac{1}{4}\kappa||\th^{*}(v_{2})-\th^{*}(v_{1})||^{2}$.
These imply that 
\[
||\nabla_{2}g(v_{1},\th^{*}(v_{2}))||\ge\frac{1}{2}\kappa||\th^{*}(v_{2})-\th^{*}(v_{1})||.
\]
Also 
\begin{align*}
 & ||\nabla_{2}g(v_{1},\th^{*}(v_{2}))||\\
= & ||\nabla_{2}g(v_{1},\th^{*}(v_{2}))-\nabla_{\th}g(v_{2},\th^{*}(v_{2}))||\\
= & ||\nabla_{2}[g(v_{1},\th^{*}(v_{2}))-g(v_{2},\th^{*}(v_{2}))]||\\
\le & ||\nabla_{[1,2]}[g(v_{1},\th^{*}(v_{2}))-g(v_{2},\th^{*}(v_{2}))]||\\
\le & L||v_{1}-v_{2}||,
\end{align*}
where $\nabla_{[1,2]}$ denotes taking the derivative on both first and second variables.
We thus conclude that 
\[
||\th^{*}(v_{2})-\th^{*}(v_{1})||\le\frac{2L}{\kappa}||v_{1}-v_{2}||.
\]

\subsubsection{Proof of Lemma \ref{lem:q smoothness}}

To prove the first property,
\begin{align*}
||\nabla_{\th}q(v,\th_{1})-\nabla_{\th}q(v,\th_{2})|| & =||\nabla_{\th}g(v,\th_{1})-\nabla_{\th}g(v,\th_{2})||\\
 & \le L||\th_{1}-\th_{2}||.
\end{align*}

Also
\begin{align*}
\left\Vert \nabla q(v_{1},\th_{1})-\nabla q(v_{2},\th_{2})\right\Vert  & =\left\Vert \nabla g(v_{1},\th_{1})-\nabla g(v_{2},\th_{2})-\nabla g(v_{1},\th^{*}(v_{1}))+\nabla g(v_{2},\th^{*}(v_{2}))\right\Vert \\
 & \le\left\Vert \nabla g(v_{1},\th_{1})-\nabla g(v_{2},\th_{2})\right\Vert +\left\Vert \nabla_{1}g(v_{1},\th^{*}(v_{1}))-\nabla_{1}g(v_{2},\th^{*}(v_{2}))\right\Vert .
\end{align*}

By Assumption \ref{asm:Smoothness} (Lipschitz continuity of $\nabla g$), 
\begin{align*}
||\nabla_{1}g(v_{1},\th^{*}(v_{1}))-\nabla_{1}g(v_{2},\th^{*}(v_{2}))|| & \le||\nabla_{[1,2]}g(v_{1},\th^{*}(v_{1}))-\nabla_{[1,2]}g(v_{2},\th^{*}(v_{2}))||\\
 & \le L\sqrt{||\th^{*}(v_{1})-\th^{*}(v_{2})||^{2}+||v_{1}-v_{2}||^{2}},
\end{align*}
where $\nabla_{[1,2]}$ denotes taking the derivative on both first
and second variable. Also By Lemma \ref{lem:th implicit lipschitz}, 
\begin{align*}
L\sqrt{||\th^{*}(v_{1})-\th^{*}(v_{2})||^{2}+||v_{1}-v_{2}||^{2}} & \le L\sqrt{\frac{4L^{2}}{\kappa^{2}}||v_{1}-v_{2}||^{2}+||v_{1}-v_{2}||^{2}}\\
 & \le L(\frac{2L}{\kappa}+1)||v_{1}-v_{2}||.
\end{align*}
This gives that 
\begin{align*}
\left\Vert \nabla q(v_{1},\th_{1})-\nabla q(v_{2},\th_{2})\right\Vert  & \le\left\Vert \nabla g(v_{1},\th_{1})-\nabla g(v_{2},\th_{2})\right\Vert +\left\Vert \nabla_{1}g(v_{1},\th^{*}(v_{1}))-\nabla_{1}g(v_{2},\th^{*}(v_{2}))\right\Vert \\
 & \le L\sqrt{||v_{1}-v_{2}||^{2}+||\th_{1}-\th_{2}||^{2}}+\left\Vert \nabla_{1}g(v_{1},\th^{*}(v_{1}))-\nabla_{1}g(v_{2},\th^{*}(v_{2}))\right\Vert \\
 & \le L\sqrt{||v_{1}-v_{2}||^{2}+||\th_{1}-\th_{2}||^{2}}+L(\frac{2L}{\kappa}+1)||v_{1}-v_{2}||\\
 & \le L_{q}\sqrt{||v_{1}-v_{2}||^{2}+||\th_{1}-\th_{2}||^{2}},
\end{align*}
where $L_{q}:=2L(L/\kappa+1)$.

\subsubsection{Proof of Lemma \ref{lem:opt th_T}}

By Lemma \ref{lem:q smoothness}, we have
\[
q(v,\th^{(t+1)})-q(v,\th^{(t)})\le-(\alpha-\frac{L\alpha^{2}}{2})||\nabla_{\th}q(v,\th^{(t)})||^{2}.
\]
By Assumption \ref{asm:PL-inequality}, we have 
\[
||\nabla_{\th}q(v,\th^{(t)})||^{2}=||\nabla_{2}g(v,\th^{(t)})||^{2}\ge\kappa(g(v,\th^{(t)})-g(v,\th^{*}(v))=\kappa q(v,\th^{(t)}).
\]
Plug-in, we have 
\[
q(v,\th^{(t+1)})\le(1-(\alpha-\frac{L\alpha^{2}}{2})\kappa)q(v,\th^{(t)}).
\]
Recursively apply this inequality, we have 
\[
q(v,\th^{(t)})\le(1-(\alpha-\frac{L\alpha^{2}}{2})\kappa)^{t}q(v,\th).
\]
Let $\bi(r,L,\kappa)=\log(1-(\alpha-L\alpha^{2}/2)\kappa)$, we have
the desired result.

\subsubsection{Proof of Lemma \ref{lem:bound d}}

Notice that $||\nabla q(v,\th)||\le||\nabla g(v,\th)||+||\nabla g(v,\th^{*}(v))||\le2M$.
$||\hdq(v,\th)||\le||\nabla_{v}g(v,\th)||+||\nabla_{1}g(v,\th^{(T)})||+||\nabla_{\th}g(v,\th)||\le3M$.
When $||\hdq||=0$, $||\d||=||\nabla f||\le M$. When $||\hdq||>0$,
\begin{align*}
||\d|| & =||[\ul||\hdq||^{2}-\left\langle \nabla f,\hdq\right\rangle ]_{+}/||\hdq||^{2}\hdq+\nabla f||\\
 & \le\ul||\hdq||+2||\nabla f||\le(2+\ul)M.
\end{align*}
This concludes that $||\d||\le(2+\ul)M$.

\subsubsection{Proof of Lemma \ref{lem:bound lambda psi}}

In the case that $\left\langle \nabla f,\hdq\right\rangle <\ul ||\hdq||^{2}$,
$\lm||\hdq||^{2}=\ul ||\hdq||^{2}-\left\langle \nabla f,\hdq\right\rangle$.
In the other case, $\lm||\hdq||^{2}=0$. Thus in all cases, 
\begin{align*}
\lm||\hdq||^{2} & \le \ul ||\hdq||^{2}+||\nabla f||\ ||\hdq||\\
 & \le \ul ||\hdq||^{2}+M||\hdq||.
\end{align*}

\subsubsection{Proof of Lemma \ref{lem:bound dq}}

Notice that since $\nabla q(v,\th^{*}(v))=0$, we have

\[
||\nabla q(v,\th)||=||\nabla q(v,\th)-\nabla q(v,\th^{*}(v))||\le L_{q}||\th-\th^{*}(v)||\le2\kappa^{-1/2}L_{q}q^{1/2}(v,\th),
\]
where the first inequality is by Lemma \ref{lem:q smoothness} and the
second inequality is by Lemma \ref{lem:Lower bound}.

\section{Proof of the Result in Section \ref{sec:kl_theory}}
We use $b$ with some subscript to denote some general $O(1)$ constant and refer reader to section \ref{misc:constant} for their detailed value.

For notation simplicity, given $v$ and $\th$, $\th^{(T)}$ denotes the results of $T$ steps of gradient of $g(v,\cdot)$ w.r.t. $\th$ starting from $\th$ using step size $\alpha$ (similar to the definition in (\ref{equ:thetaTT})). And note that $\hdq(v,\th)=\nabla g(v,\th)-\left[\nabla_{1}^{\top}g(v,\th^{(T)}),\textbf{0}^{\top}\right]^{\top}$, where $\textbf{0}$ denotes a zero vector with the same dimension as $\th$. We refer readers to the beginning of Appendix \ref{apx:pl_theory} for a discussion on the design of this extra notation and how it relates to the notation we used in Section \ref{sec:method}. For simplicity, we omit the superscript $\diamond$ in $q^{\diamond}$ and simply use $q$ to denote $q^{\diamond}$ in the proof.

We start with the following two Lemmas.

\begin{lemma} \label{lem:Lower bound local}
Under Assumption \ref{asm:KL-inequality} and assume $\alpha\le 1/L$, for any $v,\th$, $g(v,\th)-g(v,\th^{\diamond}(v,\th))\ge\frac{\kappa}{4}||\th-\th^{\diamond}(v,\th)||^{2}$.
\end{lemma}

\begin{proof}
It is easy to show that 
\[g(v,\th^{(t+1)})\le g(v,\th^{(t)})-(\alpha-\frac{L\alpha^{2}}{2})||\nabla_{\th}g(v,\th^{(t)})||^{2}\le g(v,\th^{(t)}).
\]
We thus have $g(v,\th^{\diamond}(v,\th))\le g(v,\th)$. The result of the proof follows the proof of Theorem 2 in \citet{karimi2016linear}.
\end{proof}

\begin{lemma} \label{lem:implicit smooth 2}
Under Assumption \ref{asm:Smoothness} and \ref{asm:KL-inequality}, $||\th^{\diamond}(v_{2},\th)-\th^{\diamond}(v_{1},\th)||\le\frac{4L}{\kappa}||v_{1}-v_{2}||$
for any $v_{1},v_{2}$.
\end{lemma}

\begin{proof}
Notice that $\nabla q(v_{2},\th^{\diamond}(v_{2},\th))=0$, we have
\[
||\nabla q(v_{1},\th^{\diamond}(v_{2},\th))-\nabla q(v_{2},\th^{\diamond}(v_{2},\th))||=||\nabla q(v_{1},\th^{\diamond}(v_{2},\th))||.
\]
By Assumption \ref{asm:KL-inequality}, we have $||\nabla q(v_{1},\th^{\diamond}(v_{2},\th))||\ge\sqrt{\kappa(g(v_{1},\th^{\diamond}(v_{2},\th))-g(v_{1},\th^{\diamond}(v_{1},\th))}$.
And by Lemma \ref{lem:Lower bound local}, we have 
\[
g(v_{1},\th^{\diamond}(v_{2},\th))-g(v_{1},\th^{\diamond}(v_{1},\th))\ge\frac{\kappa}{4}||\th^{\diamond}(v_{2},\th)-\th^{\diamond}(v_{1},\th)||^{2}.
\]
Combing all bounds gives that 
\[
2L||v_{1}-v_{2}||\ge||\nabla q(v_{1},\th^{\diamond}(v_{2},\th))-\nabla q(v_{2},\th^{\diamond}(v_{2},\th))||=||\nabla q(v_{1},\th^{\diamond}(v_{2},\th))||\ge\frac{\kappa}{2}||\th^{\diamond}(v_{2},\th)-\th^{\diamond}(v_{1},\th)||.
\]
This implies that $||\th^{\diamond}(v_{2},\th)-\th^{\diamond}(v_{1},\th)||\le\frac{4L}{\kappa}||v_{1}-v_{2}||.$
\end{proof}

Now we proceed to give the proof of Theorem \ref{thm:nonconvex_converge}.

Note that 
\begin{align*}
q(v_{k+1},\th_{k+1})-q(v_{k},\th_{k}) & =[g(v_{k+1},\th_{k+1})-g(v_{k+1},\th^{\diamond}(v_{k+1},\th_{k+1}))]-[g(v_{k},\th_{k})-g(v_{k},\th^{\diamond}(v_{k},\th_{k}))]\\
 & =[g(v_{k+1},\th_{k+1})-g(v_{k+1},\th^{\diamond}(v_{k+1},\th_{k}))]-[g(v_{k},\th_{k})-g(v_{k},\th^{\diamond}(v_{k},\th_{k}))]\\
 & +[g(v_{k+1},\th^{\diamond}(v_{k+1},\th_{k}))-g(v_{k+1},\th^{\diamond}(v_{k+1},\th_{k+1}))]\\
 & =[g(v_{k+1},\th_{k+1})-g(v_{k},\th_{k})]-[g(v_{k+1},\th^{\diamond}(v_{k+1},\th_{k}))-g(v_{k},\th^{\diamond}(v_{k},\th_{k}))]\\
 & +[g(v_{k+1},\th^{\diamond}(v_{k+1},\th_{k}))-g(v_{k+1},\th^{\diamond}(v_{k+1},\th_{k+1}))].
\end{align*}
Note that 
\begin{align*}
g(v_{k+1},\th_{k+1})-g(v_{k},\th_{k}) & \le-\xi\left\langle \nabla g(v_{k},\th_{k}),\d(v_{k},\th_{k})\right\rangle +\frac{L\xi^{2}}{2}||\d(v_{k},\th_{k})||^{2}\\
-[g(v_{k+1},\th^{\diamond}(v_{k+1},\th_{k}))-g(v_{k},\th^{\diamond}(v_{k},\th_{k}))] & \le\left\langle \nabla_{[1,2]} g(v_{k},\th^{\diamond}(v_{k},\th_{k})),[v_{k+1},\th^{\diamond}(v_{k+1},\th_{k})]-[v_{k},\th^{\diamond}(v_{k},\th_{k})]\right\rangle \\
 & +\frac{L}{2}||[v_{k+1},\th^{\diamond}(v_{k+1},\th_{k})]-[v_{k},\th^{\diamond}(v_{k},\th_{k})]||^{2}.
\end{align*}
Notice that as $\nabla_{2}g(v_{k},\th^{\diamond}(v_{k},\th_{k}))=0$, 
\[\left\langle \nabla_{[1,2]} g(v_{k},\th^{\diamond}(v_{k},\th_{k})),[v_{k+1},\th^{\diamond}(v_{k+1},\th_{k})]-[v_{k},\th^{\diamond}(v_{k},\th_{k})]\right\rangle =\xi\left\langle \nabla_{[1,2]} g(v_{k},\th^{\diamond}(v_{k},\th_{k})),\d(v_{k},\th_{k})\right\rangle.
\]
Also using Lemma \ref{lem:implicit smooth 2}, we have
\[
||\th^{\diamond}(v_{k+1},\th_{k})-\th^{\diamond}(v_{k},\th_{k})||\le\frac{4L}{\kappa}||v_{k+1}-v_{k}||.
\]
This implies that 
\[
||[v_{k+1},\th^{\diamond}(v_{k+1},\th_{k})]-[v_{k},\th^{\diamond}(v_{k},\th_{k})]||^{2}\le(\frac{16L^{2}}{\kappa^{2}}+1)||v_{k+1}-v_{k}||^{2}\le(\frac{16L^{2}}{\kappa^{2}}+1)\xi^{2}||\d(v_{k},\th_{k})||^{2}.
\]
We thus have 
\[
q(v_{k+1},\th_{k+1})-q(v_{k},\th_{k})\le-\xi\left\langle \nabla q(v_{k},\th_{k}),\d(v_{k},\th_{k})\right\rangle +L_{q}\xi^{2}||\d(v_{k},\th_{k})||^{2}/2+\chi_{k},
\]
where we define $L_{q}=(\frac{16L^{2}}{\kappa^{2}}+2)$ and $\chi_{k}=[g(v_{k+1},\th^{\diamond}(v_{k+1},\th_{k}))-g(v_{k+1},\th^{\diamond}(v_{k+1},\th_{k+1}))].$
Using the same argument in the proof of Lemma \ref{lem:decay q} and Lemma \ref{lem:decay hdq}, we
have
\begin{align*}
 & q(v_{k+1},\th_{k+1})-q(v_{k},\th_{k})\\
\le & -\xi \ul ||\nabla q(v_{k},\th_{k})||^{2}+12\xi \ul L_{q}^{2}\kappa^{-1}\exp(-\bi T)q(v_{k},\th_{k})\\
+ & 2\xi\bii L\kappa^{-1/2}\exp(-\bi T/2)q^{1/2}(v_{k},\th_{k})+L_{q}\xi^{2}\bii^{2}/2+\chi_{k}\\
\le & -\xi \ul ||\nabla q(v_{k},\th_{k})||^{2}+12\xi \ul L_{q}^{2}\kappa^{-2}\exp(-\bi T)||\nabla q(v_{k},\th_{k})||^{2}\\
+ & 2\xi\bii L\kappa^{-1}\exp(-\bi T/2)||\nabla q(v_{k},\th_{k})||+L_{q}\xi^{2}\bii^{2}/2+\chi_{k}.
\end{align*}
Here the second inequality is by Assumption \ref{asm:KL-inequality}. Choosing
$T$ such that $T\ge\bviii(\ul ,\alpha,\kappa,L)$ where
\[
\bviii(\ul ,\alpha,\kappa,L)=\left\lceil -b_{1}^{-1}\log(\frac{\kappa^{2}}{48\ul L_{q}^{2}})\right\rceil,
\]
we have 
\[
q(v_{k+1},\th_{k+1})-q(v_{k},\th_{k})\le-\frac{3}{4}\xi \ul ||\nabla q(v_{k},\th_{k})||^{2}+2\xi\bii L\kappa^{-1}\exp(-\bi T/2)||\nabla q(v_{k},\th_{k})||+L_{q}\xi^{2}\bii^{2}/2+\chi_{k}.
\]
Using Young's inequality, given any $x>0$,
\[
\exp(-\bi T/2)||\nabla q(v_{k},\th_{k})||\le x\exp(-\bi T)+\frac{1}{x}||\nabla q(v_{k},\th_{k})||^{2}.
\]
Choosing $x=\frac{4L\bii}{\ul \kappa}$, we have
\[
q(v_{k+1},\th_{k+1})-q(v_{k},\th_{k})\le-\frac{1}{4}\xi \ul ||\nabla q(v_{k},\th_{k})||^{2}+\Delta+\chi_{k},
\]
where we denote $\Delta=\xi\frac{8L^{2}\bii^{2}}{\ul \kappa^{2}}\exp(-\bi T)+\frac{1}{2}L_{q}\xi^{2}\bii^{2}$.
This gives that 
\[
\frac{1}{4}\xi \ul \sum_{k=0}^{K}||\nabla q(v_{k},\th_{k})||^{2}\le q(v_{0},\th_{0})-q(v_{K},\th_{K})+K\Delta+\sum_{k=0}^{K-1}\chi_{k}.
\]
Using the same argument in the proof of Lemma \ref{lem:decay hdq}, 
\[
||\hdq(v,\th)||\le(2L\kappa^{-1}+1)||\nabla q(v,\th)||.
\]
We hence have 
\begin{align*}
\sum_{k=0}^{K-1}||\hdq(v_{k},\th_{k})||^{2} & \le(2L\kappa^{-1}+1)^{2}\sum_{k=0}^{K-1}||\nabla q(v_{k},\th_{k})||^{2}\\
 & \le\frac{4(2L\kappa^{-1}+1)^{2}}{\xi \ul }(q(v_{0},\th_{0})-q(v_{K},\th_{K})+K\Delta+\sum_{k=0}^{K-1}\chi_{k}).
\end{align*}
Similar to the proof of Lemma
\ref{lem:Convergence hk}, 
\[
\sum_{k=0}^{K-1}||\d(v_{k},\th_{k})||^{2}\le\frac{2(f(v_{0},\th_{0})-f(v_{K},\th_{K}))}{\xi}+2\ul ^{2}\sum_{k=0}^{K-1}||\hdq(v_{k},\th_{k})||^{2}+2\ul M\sqrt{K}\sqrt{\sum_{k=0}^{K-1}||\hdq(v_{k},\th_{k})||^{2}}.
\]
Using$\sqrt{x+y}\le\sqrt{x}+\sqrt{y}$, we have
\begin{align*}
2\ul ^{2}\sum_{k=0}^{K-1}||\hdq(v_{k},\th_{k})||^{2} & \le\frac{8\ul (2L\kappa^{-1}+1)^{2}}{\xi }(q(v_{0},\th_{0})-q(v_{K},\th_{K})+K\Delta+\sum_{k=0}^{K-1}\chi_{k})\\
2\ul M\sqrt{K}\sqrt{\sum_{k=0}^{K-1}||\hdq(v_{k},\th_{k})||^{2}} & \le\sqrt{K}\frac{4\ul^{1/2} M(2L\kappa^{-1}+1)}{\xi^{1/2}}(\sqrt{q(v_{0},\th_{0})-q(v_{K},\th_{K})}+K^{1/2}\Delta^{1/2}+\sqrt{\left[\sum_{k=0}^{K-1}\chi_{k}\right]_{+}}).
\end{align*}
Also notice that by Assumption \ref{asm:KL-inequality},
\begin{align*}
\sum_{k=0}^{K-1}q(v_{k},\th_{k}) & \le\sum_{k=0}^{K-1}\frac{\xi}{\kappa}||\nabla q(v_{k},\th_{k})||^{2}\\
 & \le\frac{4}{\ul \kappa\xi}(q(v_{0},\th_{0})-q(v_{K},\th_{K})+K\Delta+\sum_{k=0}^{K-1}\chi_{k})
\end{align*}
We hence have
\begin{align*}
\sum_{k=0}^{K-1}(||\d(v_{k},\th_{k})||^{2}+q(v_{k},\th_{k})) & =O\left(\frac{1}{\xi}+\frac{K\Delta}{\xi}+\frac{K^{1/2}}{\xi^{1/2}}+\frac{K\Delta^{1/2}}{\xi^{1/2}}+K^{1/2}\left(\left[\sum_{k=0}^{K-1}\chi_{k}\right]_+\right)^{1/2}\right)\\
 & =O\left(\frac{1}{\xi}+K\exp(-\bi T/2)+K\xi^{1/2}+\frac{K^{1/2}}{\xi^{1/2}}+\left(K\left[\sum_{k=0}^{K-1}\chi_{k}\right]_+\right)^{1/2}\right).
\end{align*}
Using the same argument as the proof of Theorem \ref{thm:Convergence k}, when $T\ge\left\lceil -b_{1}^{-1}\log(\frac{1}{16}\kappa^2 L^{-2})\right\rceil $,
\[
\K^{\diamond}(v,\th)\le2||\d(v,\th)||^{2}+q(v,\th)+8L^{2}\exp(-\bi T)\kappa^{-2}(\ul +2)^{2}\bii^{2}.
\]
This implies that 
\begin{align*}
\min_{k}\K^{\diamond}(v_{k},\th_{k}) & \le\frac{1}{K}\sum_{k=0}^{K-1}[2||\d(v,\th)||^{2}+q(v,\th)]+8L^{2}\exp(-\bi T)\kappa^{-2}(\ul +2)^{2}\bii^{2}\\
 & =O\left(\frac{1}{\xi K}+\exp(-\bi T/2)+\xi^{1/2}+\frac{1}{\xi^{1/2}K^{1/2}}+\left(\left[\frac{1}{K}\sum_{k=0}^{K-1}\chi_{k}\right]_+\right)^{1/2}\right).
\end{align*}

Now we proceed to bound $\frac{1}{K}\sum_{k=0}^{K-1}\chi_{k}$. Notice
that 
\begin{align*}
\chi_{k} & =g(v_{k+1},\th^{\diamond}(v_{k+1},\th_{k}))-g(v_{k+1},\th^{\diamond}(v_{k+1},\th_{k+1}))\\
 & =g(v_{k+1},\th^{\diamond}(v_{k+1},\th_{k}))-g(v_{k},\th^{\diamond}(v_{k},\th_{k}))+g(v_{k},\th^{\diamond}(v_{k},\th_{k}))-g(v_{k+1},\th^{\diamond}(v_{k+1},\th_{k+1})).
\end{align*}
Notice that using Assumption \ref{asm:Smoothness} and Lemma \ref{lem:implicit smooth 2}
\begin{align*}
g(v_{k+1},\th^{\diamond}(v_{k+1},\th_{k}))-g(v_{k},\th^{\diamond}(v_{k},\th_{k})) & \le L||[v_{k+1},\th^{\diamond}(v_{k+1},\th_{k})]-[v_{k},\th^{\diamond}(v_{k},\th_{k})]||\\
 & \le L(||v_{k+1}-v_{k}||+||\th^{\diamond}(v_{k+1},\th_{k})-\th^{\diamond}(v_{k},\th_{k})||)\\
 & \le(L+\frac{4L}{\kappa})||v_{k+1}-v_{k}||\\
 & \le(L+\frac{4L}{\kappa})\xi||\d(v_{k},\th_{k})||.
\end{align*}
Note that using the same procedure as the proof of Lemma \ref{lem:bound d},
$||\d(v_{k},\th_{k})||\le\bii$. We thus conclude that 
\begin{align*}
\sum_{k=0}^{K-1}\chi_{k} & \le\sum_{k=0}^{K-1}g(v_{k},\th^{\diamond}(v_{k},\th_{k}))-g(v_{k+1},\th^{\diamond}(v_{k+1},\th_{k+1}))\\
 & +(L+\frac{4L}{\kappa})\xi\sum_{k=0}^{K-1}||\d(v_{k},\th_{k})||\\
 & \le\sum_{k=0}^{K-1}g(v_{k},\th^{\diamond}(v_{k},\th_{k}))-g(v_{k+1},\th^{\diamond}(v_{k+1},\th_{k+1}))+(L+\frac{4L}{\kappa})\bii\xi K\\
 & =g(v_{0},\th^{\diamond}(v_{0},\th_{0}))-g(v_{K},\th^{\diamond}(v_{K},\th_{K}))+(L+\frac{4L}{\kappa})\bii\xi K.
\end{align*}
We thus have $\frac{1}{K}\sum_{k=0}^{K-1}\chi_{k}=O(\frac{1}{K}+\xi).$

\section{List of absolute constants used in the proofs} \label{misc:constant}
Here we summarize the absolute constant used in the proofs.
\begin{align*}
b_{1}(\alpha,L,\kappa) & =\log(1-(\alpha-L\alpha^{2}/2)\kappa)\\
\bii(M,\ul ) & =(3+\ul )M\\
\biii(\ul ,\alpha,\kappa,L) & =\left\lceil -b_{1}^{-1}\log(\frac{\ul \kappa}{64\ul L_{q}^{2}})\right\rceil \\
\biv(\ul ,\kappa,\xi) & =-\log(1-\frac{\xi}{4}\ul \kappa)\\
\bv(L_{q},\ul ,\kappa) & =\frac{16L_{q}^{2}}{\ul \kappa^{2}}(2L\kappa^{-1}+1)^{2}\\
\bvi(\kappa,L) & =(2L\kappa^{-1}+1)^{2}\\
\bviii(\ul ,\alpha,\kappa,L) & =\left\lceil -b_{1}^{-1}\log(\frac{\kappa^{2}}{48\ul L_{q}^{2}})\right\rceil 
\end{align*}

\end{document}